\documentclass[twoside]{article}

\usepackage[preprint]{aistats2026}
\usepackage[round,sort&compress]{natbib}
\usepackage[utf8]{inputenc}
\usepackage{amsmath}
\usepackage{color}
\usepackage{xcolor}
\usepackage{bbm}
\usepackage{breakcites}
\usepackage{thm-restate}

\usepackage{hyperref}
\hypersetup{
    colorlinks=false,
    linkcolor=blue,
    filecolor=magenta,      
    urlcolor=blue,
    pdfpagemode=FullScreen,
    }

\newcommand{\cX}{\mathcal{X}}

\def\E{\mathbb{E}}

\def\1{\mathbbm{1}}

\usepackage{amsmath}
\usepackage{amssymb}
\usepackage{mathtools}
\usepackage{amsthm}
\theoremstyle{plain}
\newtheorem{theorem}{Theorem}[section]
\newtheorem{proposition}[theorem]{Proposition}
\newtheorem{lemma}[theorem]{Lemma}

\theoremstyle{definition}

\theoremstyle{remark}

 \usepackage{algorithm}
\usepackage[ruled,lined,algo2e,noend, linesnumbered]{algorithm2e}
\SetKwInput{KwInit}{Initialization}
\allowdisplaybreaks
\begin{document}

%
\runningtitle{Last-Iterate Convergence in Bilinear Saddle-Point Problems under Bandit Feedback}

%
\runningauthor{Maiti*, Zhang*, Jamieson, Morgenstern, Panageas, Ratliff}

\newcommand{\equalcontrib}{\textsuperscript{*}}

\twocolumn[

\aistatstitle{Efficient Uncoupled Learning Dynamics with $\tilde{O}\!\left(T^{-1/4}\right)$ Last-Iterate Convergence in Bilinear Saddle-Point Problems over Convex Sets under Bandit Feedback}

\aistatsauthor{
Arnab Maiti\equalcontrib\textsuperscript{1} \And
Claire Jie Zhang\equalcontrib\textsuperscript{1} \And
Kevin Jamieson\textsuperscript{1} \AND
Jamie Heather Morgenstern\textsuperscript{1,3} \And
Ioannis Panageas\textsuperscript{2} \And
Lillian J. Ratliff\textsuperscript{1}
}

\aistatsaddress{
\textsuperscript{1}University of Washington \And
\textsuperscript{2}University of California, Irvine \And
\textsuperscript{3}Amazon
}


]

\begin{abstract}
  
  In this paper, we study last-iterate convergence of learning algorithms in bilinear saddle-point problems, a preferable notion of convergence that captures the day-to-day behavior of learning dynamics. We focus on the challenging setting where players select actions from compact convex sets and receive only bandit feedback.  Our main contribution is the design of an uncoupled learning algorithm that guarantees last-iterate convergence to the Nash equilibrium with high probability. We establish a convergence rate of $\tilde{O}(T^{-1/4})$ up to polynomial factors in problem parameters. Crucially, our proposed algorithm is computationally efficient, requiring only an efficient linear optimization oracle over the players' compact action sets. The algorithm is obtained by combining techniques from experimental design and the classic Follow-The-Regularized-Leader (FTRL) framework, with a carefully chosen regularizer function tailored to the geometry of the action set of each learner. 

\end{abstract}

\section{INTRODUCTION}
Online learning in games is a well-studied area (\cite{syrgkanis2015fast, chen2020hedging, daskalakis2021near, anagnostides2022uncoupled}) that investigates the convergence properties of learning algorithms in game-theoretic settings. This line of research has been instrumental in developing superhuman AI agents for competitive environments such as Go \citep{silver2017mastering}, Poker \citep{brown2018superhuman} and Diplomacy \citep{meta2022human}. It is well known that standard algorithms such as Follow-the-Regularized-Leader (FTRL) and Mirror Descent converge to a Nash equilibrium in the average-iterate sense under self-play. In other words, while individual strategies may remain far from equilibrium, their average converges to the Nash equilibrium. The seminal works of \cite{daskalakis2011near} and \cite{rakhlin2013online} further strengthened this understanding by establishing near-optimal convergence rates in the average-iterate sense. However, works by \cite{mertikopoulos2018cycles,bailey2018multiplicative} showed that many standard algorithms that succeed in the average-iterate sense fail to converge in the last-iterate sense, which is often more desirable in practice as it reflects the day-to-day behavior of the learners.

Motivated by this negative result, a new line of work has focused on designing uncoupled learning algorithms that achieve last-iterate convergence to a Nash equilibrium in self-play \cite{daskalakis2018last,cai2022finite,cai2024fast}. In particular, optimistic variants of classical algorithms have been shown to exhibit last-iterate convergence under the gradient feedback setting \cite{daskalakis2017training,liang2019interaction,wei2020linear}. Moreover, \cite{wei2020linear} established last-iterate linear convergence for bilinear games with polytope action sets under gradient feedback.

Algorithms under bandit feedback-where only the payoff of the chosen action is observed, in contrast to the richer gradient feedback-form a well-studied area in the multi-armed bandits literature due to their practical relevance \cite{auer2002nonstochastic,bubeck2012towards,neu2015explore,zimmert2022return}, but results on last-iterate convergence remain relatively sparse. Under a variant of the standard bandit feedback model, \cite{cai2023uncoupled} first showed convergence to a Nash equilibrium with high probability at a rate of $T^{-1/8}$ in matrix games, later improved to $T^{-1/5}$ in \cite{cai2025average}.

While classical game theory has deep roots in discrete actions, many modern strategic interactions are inherently continuous. Players often select from a continuum of strategies rather than a finite list, as in applications such as algorithmic pricing, resource allocation, routing and multi-agent robotics \citep{besbes2009dynamic, den2015dynamic, krichene2015online}. Related ideas also appear in the alignment of language models \citep{munos2023nash}. These settings are formally captured by compact convex action sets, for which no high-probability last-iterate guarantees under standard bandit feedback are currently known. The only established result under the standard bandit feedback model is due to \cite{dong2024uncoupled}, who proposed an uncoupled learning algorithm whose iterates converge to a Nash equilibrium only in expectation at a rate of $T^{-1/6}$.

This gap motivates our central question:

\begin{quote}
\textit{Given a bilinear function with compact and convex action sets, does there exist an uncoupled learning algorithm whose iterates converge to a Nash equilibrium with high probability in the self-play setting, under bandit feedback?} 
\end{quote}

\subsection{Problem Setting}
In this paper, we answer the above question in the affirmative. To this end, we formalize the setting of last-iterate convergence in bilinear saddle-point problems under bandit feedback. Let $\mathcal{X} \subset \mathbb{R}^{n}$ and $\mathcal{Y} \subset \mathbb{R}^{m}$ be compact, convex sets, and let $A \in \mathbb{R}^{n \times m}$ be an input matrix. We assume $\mathrm{span}(\mathcal{X}) = \mathbb{R}^{n}$ and $\mathrm{span}(\mathcal{Y}) = \mathbb{R}^{m}$. For simplicity of presentation, throughout this paper we also assume that $\langle x, Ay \rangle \in [-1,1]$ for all $x \in \mathcal{X}$ and $y \in \mathcal{Y}$. 

In each round $k$, the row player selects $x_k \in \mathcal{X}$ and the column player selects $y_k \in \mathcal{Y}$. They then receive standard bandit feedback in the form of $\langle x_k, Ay_k \rangle$ and $-\langle x_k, Ay_k \rangle$, respectively. A variant of this feedback was studied by \cite{cai2023uncoupled,cai2022finite} for probability simplices, where $i_k \sim x_k$ and $j_k \sim y_k$ are sampled, and the players observe only $A_{i_k,j_k}$ and $-A_{i_k,j_k}$. Even when $\mathcal{X}$ and $\mathcal{Y}$ are probability simplices, the two feedback types are fundamentally different, and the results are not directly comparable.

We focus on \emph{uncoupled learning algorithms}, which operate entirely on a player’s own action set and make no assumptions about the opponent: they do not observe the opponent’s actions, do not know the opponent’s action set, and not even the dimension of the action set. The goal is to design such algorithms for both players under standard bandit feedback so that the pair $(x_k,y_k)$ forms an $\varepsilon_k$-approximate Nash equilibrium (last-iterate convergence) with high probability, where $\varepsilon_k$ depends polynomially on $n$ and $m$ and satisfies $\lim_{k \to \infty} \varepsilon_k = 0$.  

Recall that a pair $(\tilde x,\tilde y)$ is an $\varepsilon$-approximate Nash equilibrium if, for all $(x,y) \in \mathcal{X} \times \mathcal{Y}$,
\[
    \langle x, A \tilde y \rangle - \varepsilon 
    \;\leq\; \langle \tilde x, A \tilde y \rangle 
    \;\leq\; \langle \tilde x, Ay \rangle + \varepsilon.
\]

If only $(\mathbb{E}[x_k], \mathbb{E}[y_k])$ can be shown to form an $\varepsilon_k$-approximate Nash equilibrium, then the convergence is said to hold only in expectation, as in \cite{dong2024uncoupled}. Such convergence guarantees are weaker, since they do not ensure convergence along a single trajectory and may require multiple runs to learn an equilibrium, which is often undesirable in practice.



\subsection{Contributions}
In this paper, we design the first uncoupled learning dynamics whose iterates exhibit last-iterate convergence with high probability under standard bandit feedback for bilinear saddle-point problems over convex sets. Formally, we construct uncoupled learning dynamics such that the pair $(x_k,y_k)$ forms an $\varepsilon_k$-approximate Nash equilibrium with probability at least $1-\delta$, where
$$
\varepsilon_k = \mathrm{poly}(n,m,\log(k/\delta))\,k^{-1/4}.
$$
This result also improves upon the $k^{-1/6}$ convergence rate of \cite{dong2024uncoupled}, who established last-iterate convergence only in expectation. Moreover, if the action sets admit efficient linear optimization oracles, our dynamics can be implemented in polynomial time.  

Our approach builds on the average-to-last-iterate framework recently introduced by \cite{cai2025average}. The key challenge in adapting this framework to our setting is that, unlike \cite{cai2025average}, who worked with probability simplices under a variant of bandit feedback where estimating rewards is relatively straightforward and the negative entropy regularizer is a natural choice (well aligned with the $(\|\cdot\|_1,\|\cdot\|_\infty)$ primal-dual pair), we consider arbitrary compact convex sets under standard bandit feedback. This makes reward estimation significantly more challenging, since strategies must remain approximate equilibria and we cannot freely explore suboptimal strategies. We address this difficulty through a carefully designed sampling procedure that leverages experimental design techniques from linear bandits.  

In addition, we construct regularizers tailored to the geometry of the action sets. This is necessary to ensure that the regularizer is compatible with the norms naturally arising from our estimation guarantees. Finally, all of these steps must be carried out while preserving computational efficiency whenever the action sets admit efficient linear optimization oracles, which is ensured by our algorithm.
\subsection{Related Work}
\cite{muthukumar2020impossibility} ruled out last-iterate convergence for certain well-known classes of uncoupled learning dynamics that would have served as natural candidates under the bandit feedback setting. The first results on last-iterate convergence rates for uncoupled learning dynamics in two-player zero-sum games under a variant of standard bandit feedback were presented by \cite{cai2023uncoupled}. They showed that simple uncoupled dynamics based on mirror descent with KL-divergence, combined with carefully chosen subsets of action sets and suitable loss estimators, achieve a last-iterate convergence rate of $\tilde{O}(T^{-1/8})$ with high probability and $\tilde{O}(T^{-1/6})$ in expectation. In the same work, they also generalized their results to Markov games. The high-probability rate was later improved to $\tilde{O}(T^{-1/5})$ by \cite{cai2025average}, while a concurrent work by \cite{fiegel2025harder} established a lower bound of $\Omega(T^{-1/3})$ for this setting. Recently, \cite{chen2023finite,chen2024last} proposed smoothed best-response dynamics for two-player zero-sum stochastic games.  

In the bilinear setting under standard bandit feedback, \cite{dong2024uncoupled} introduced mirror descent based uncoupled dynamics with appropriate gradient estimators, achieving an $O(T^{-1/6})$ last-iterate convergence rate, though only in expectation. In a broader class of monotone games, \cite{tatarenko2019learning} established asymptotic last-iterate convergence to Nash equilibrium, albeit without finite-time guarantees.  

The literature on learning in games is extensive. Here, we primarily focused on works concerning last-iterate convergence under bandit feedback in two-player zero-sum games. For results on average-iterate convergence, we refer the reader to \cite{daskalakis2011near,syrgkanis2015fast,rakhlin2013online,chen2020hedging,daskalakis2021near,anagnostides2022uncoupled} and the references therein. For results on last-iterate convergence under gradient feedback, see \cite{daskalakis2017training, liang2019interaction,  wei2020linear, daskalakis2018last,anagnostides2022last, abe2024boosting,cai2025average} and the references therein. For other conditions such as strict equilibria and strong monotonicity, we refer the reader to \cite{giannou2021rate,jordan2025adaptive,ba2025doubly} and the references therein.
\section{ALGORITHM WITH LAST-ITERATE CONVERGENCE
}
Recently, \cite{cai2025average} introduced a framework for zero-sum games over probability simplices that transforms uncoupled dynamics with average-iterate convergence guarantees into ones with last-iterate convergence guarantees. The framework runs an average-iterate algorithm over multiple phases, where a phase $t$ consists of $B_t$ rounds. In a phase $t$, if the average-iterate algorithm outputs $\tilde{x}_t$, the framework plays strategies $x_k$ close to $\bar{x}_t := \tfrac{1}{t}\sum_{s=1}^t \tilde{x}_s$ for each round $k$ in that phase. These strategies are then used to estimate $A\widehat{y}_t$, where $\widehat{y}_t$ denotes the expected strategy of the other player in phase $t$. This estimate defines a phase utility vector that, when fed back into the average-iterate algorithm, drives $\bar{x}_t$ toward equilibrium. Since the framework plays strategies near $\bar{x}_t$, last-iterate convergence is achieved.


We adapt this framework to our setting in order to achieve last-iterate convergence, with details given in Algorithm \ref{algo:last_iter_bandit_feedback}. The main challenge lies in sampling in each phase so as to estimate $A\widehat{y}_t$ accurately with respect to a suitable dual norm, and in constructing a regularizer that is strongly convex with respect to a corresponding primal norm while maintaining low Bregman divergence. Our approaches to these challenges are presented in Sections \ref{sec:estimator} and \ref{subsec:regularizer_intro}, which form the core technical contributions of this paper. We state our main result in Section \ref{sec:main-result}.

\begin{algorithm}
\caption{Last-iterate algorithm for the row player under bandit feedback}\label{algo:last_iter_bandit_feedback}

\KwIn{Probability error term $\delta \in (0,\tfrac{1}{2}]$, step size $\eta > 0$, batch size $B_t \gets \log(8t^2/\delta)\cdot t^{3}$, mixing parameter $\lambda_t \gets t^{-2}$, exploration distribution $\mathcal{D}_\mathcal{X}$ over $\mathcal{X}$.}

\KwInit{round counter $k \gets 1$; $\widehat{\theta}_{0}^x \gets \mathbf{0}$;\\ $\qquad\qquad\qquad \quad x_{1} \gets \mathbb{E}_{x\sim \mathcal{D}}[x]$.}

\For{phase $t=1,2,\ldots$}{
    Compute running average: $\bar{x}_t \gets \tfrac{1}{t}\sum_{\ell=1}^{t}\tilde{x}_{\ell}$\;
    
    \For{$s=1,2,\ldots,B_t$}{
        With probability $1/2$, set $x_{t,s} \gets \bar{x}_t$\;
        
        With probability $1/2$, sample $z_{t,s} \sim \mathcal{D}_\mathcal{X}$ and set  
        $x_{t,s} \gets (1-\lambda_t)\bar{x}_t + \lambda_t z_{t,s}$\;
        
        Play strategy $x_k = x_{t,s}$, observe reward $r_{t,s}$, and update $k \gets k+1$\;
    }
    
    Construct estimate $\widehat{\theta}_t^x$ of the mean reward vector $\bar{\theta}_t^x$  
    using $(x_{t,s}, r_{t,s})$ as described in Section~\ref{sec:estimator}\;
    
    Estimate phase utility: $\widehat{u}_t^x \gets t \cdot \widehat{\theta}_t^x - (t-1)\cdot \widehat{\theta}_{t-1}^x$\;
    
    Update via \textsc{OFTRL} with regularizer $\phi(x)$ from Section~\ref{subsec:regularizer_intro}:
    \[
      \tilde{x}_{t+1} \gets \arg\max_{x\in \mathcal{X}}
      \left\{ \left\langle x, \sum_{\ell=1}^t \widehat{u}_{\ell} + \widehat{u}_{t}\right\rangle
      - \tfrac{1}{\eta}\,\phi(x) \right\}.
    \]
}
\end{algorithm}

\subsection{Sampling method and estimator}\label{sec:estimator}
Analogous to the row player's algorithm, we can describe an algorithm for the column player, where in the $s$-th round of the $t$-th phase it selects $y_{t,s}$. Denote the row player’s true expected utility vector as
\[
\bar{\theta}_t^x := A \widehat{y}_{t},
\]
where $\widehat{y}_{t} := \mathbb{E}[y_{t,s}] = \frac{1}{2}  \big( (1-\lambda_t)\bar{y}_t + \lambda_t \E_{z\sim \mathcal{D}_\mathcal{Y}}[z] \big) + \frac{1}{2}\bar{y}_t$ is the column player’s averaged strategy in phase $t$, where $\mathcal{D}_{\mathcal{Y}}$ is the exploration distribution of the column player.

In this section, we describe how to construct an estimator $\widehat{\theta}_t^x$ of the utility vector $\bar{\theta}_t^x$ and establish meaningful concentration guarantees. An analogous estimator can be constructed for the column player.

We now begin the construction. Recall that in each round $s$ of phase $t$, the row player plays $x_{t,s} \in \mathcal{X}$ and receives the reward
\[
r_{t,s} = \langle x_{t,s}, Ay_{t,s}\rangle.
\]
We can decompose this reward as
\[
r_{t,s} = \langle x_{t,s}, A\widehat{y}_{t}\rangle + \langle x_{t,s}, A(y_{t,s} - \widehat{y}_t)\rangle.
\]
This yields a linear model, where the second term $\langle x_{t,s}, A(y_{t,s} - \widehat{y}_t)\rangle$ is zero-mean $4\lambda_t^2$-subgaussian noise in each phase $t$ (which we show in Appendix \ref{appendix:estimation}).  

For simplicity of exposition, assume that in each phase $t$, half of the $B_t$ rounds use $x_{t,s} = \bar{x}_t$, where we denote these indices by $\{s_1, s_2, \ldots, s_{B_t/2}\}$. In the remaining $B_t/2$ rounds, we set
\[
x_{t,s} \gets (1-\lambda_t)\bar{x}_t + \lambda_t z_{t,s}, \qquad z_{t,s} \sim \mathcal{D},
\]
and denote the corresponding indices by $\{s_1', s_2', \ldots, s_{B_t/2}'\}$. We address all the other possible cases in Appendix \ref{appendix-sec:estimate-exp-dist}. 

We now construct pairs $(s_i, s_i')$ such that
\[
x_{t,s_i'} = (1-\lambda_t)x_{t,s_i} + \lambda_t z_{t,s_i'}.
\]

Now consider the transformed reward
\[
\widehat{r}_{t,s_i'} := \frac{r_{t,s_i'} - (1-\lambda_t)r_{t,s_i}}{\lambda_t}
= \langle z_{t,s_i'}, \theta_t \rangle + \widehat{\eta}_{t,s_i'},
\]
where $\widehat{\eta}_{t,s_i'}$ is zero-mean $8$-subgaussian noise.

Thus, from the pairs $(z_{t,s_i'}, \widehat{r}_{t,s_i'})$, we obtain an unbiased estimator of $\bar\theta_t^x$:
\[
\widehat{\theta}_t^x
= \left(\sum_{i=1}^{B_t/2} z_{t,s_i'} z_{t,s_i'}^\top\right)^{-1}
  \sum_{i=1}^{B_t/2} \widehat{r}_{t,s_i'} z_{t,s_i'}.
\]

Finally, if the vectors $z_{t,s_i'}$ are sampled from an exploration distribution $\mathcal{D}_{\mathcal{X}}$ that is uniform over a subset $S := \{x_1, \dots, x_n\} \subset \mathcal{X}$ satisfying
$$
\sup_{x \in \mathcal{X}} x^\top V^{-1} x \;\le\; 2n^2,
\qquad
V := \frac{1}{n} \sum_{i=1}^n x_i x_i^\top \ \succ 0,
$$
then we obtain meaningful concentration guarantees, which are used in our analysis.  
We formalize this in the following lemma, with the proof provided in Appendix \ref{appendix-sec:estimator}.
\begin{lemma}[Estimator concentration bound]\label{lem:concentration-theta} The estimator $\widehat{\theta}_t^x$ constructed in each phase $t$ satisfies the following:
\begin{equation*}
\Pr\!\left(
\sup_{x\in\mathcal{X}} |\langle x,\widehat\theta_t^x-\bar{\theta}_t^x\rangle|
\;\le\;
48\sqrt{\tfrac{n^3}{t^3}}
\right)
\;\ge\; 1-\delta/(4t^2) .
\end{equation*}
\end{lemma}
\textbf{Remark:} In the previous work of \cite{cai2025average}, where the action sets are probability simplices, after choosing $x_k$ the algorithm is allowed to sample $i_k \sim x_k$ and observe the corresponding reward, which simplifies the estimation process. In our setting, on the other hand, we only observe the reward of the actual strategy $\langle x_{t,s}, Ay_{t,s}\rangle$. Thus, estimating $\bar{\theta}_t^x$ requires the transformation described in this section, and is made possible by the specific sampling scheme used in each phase.
\subsection{Choice of regularizer and the corresponding primal--dual norms}\label{subsec:regularizer_intro}
Recall that our algorithm updates the players' strategies using the OFTRL framework. The efficiency of OFTRL critically depends on the choice of the regularizer function $\phi(x)$. An ideal regularizer should be strongly convex and have a small diameter over the action set $\mathcal{X}$. To achieve this, we tailor the regularizer to the geometry of $\mathcal{X}$.  

We begin by defining a pair of primal-dual norms in Section \ref{sec:norms}, intrinsically tied to the action set through its symmetrization $K := \operatorname{conv}(\mathcal{X} \cup (-\mathcal{X}))$. We then construct an ellipsoid $E = \{x : x^\top H x \leq 1\}$, which serves as a tight approximation of $K$ up to polynomial factors in the dimension. The regularizer is chosen to be half the squared norm induced by this ellipsoid, namely $\phi(x) := \tfrac{1}{2} x^\top H x$. As we show in Section \ref{sec:regularizer}, this choice yields a regularizer that is $1$-strongly convex with respect to the primal norm and whose Bregman divergence scales polynomially with the dimension.




\subsubsection{Primal-dual norm pair}\label{sec:norms}
In this section, we formally establish that the norms tailored to the action set $\mathcal{X}$ constitute a valid primal-dual pair. Analogous norms can be defined for the action set $\mathcal{Y}$.

Let $\|z\|_{*,\mathcal{X}}:=\max_{x\in\mathcal{X}}|\langle x,z\rangle|$ and $\|z\|_{\mathcal{X}}:=\max_{\|y\|_{*,\mathcal{X}}\leq 1}\langle y,z\rangle$. 
Recall that $\mathcal{X}$ spans $\mathbb{R}^d$. Therefore, one can establish the following properties, proof of which is provided in the Appendix \ref{appendix:norms} for completeness.
\begin{lemma}[\cite{chandrasekaran2012convex}]
    The following properties hold for $\|\cdot\|_\mathcal{X}$ and $\|\cdot\|_{*,\mathcal{X}}$:
    \begin{itemize}
        \item $\|\cdot\|_\mathcal{X}$ and $\|\cdot\|_{*,\mathcal{X}}$ are both norms.
        \item $\|\cdot\|_{*,\mathcal{X}}$ is the dual norm of $\|\cdot\|_\mathcal{X}$
        \item $\{y:\|y\|_{\mathcal{X}}\le 1\}=\mathrm{conv}(\mathcal{X}\cup-\mathcal{X})$ 
    \end{itemize}
\end{lemma}

\textbf{Remark:} This choice of primal-dual norms is mainly motivated by the following reasoning. For any choice of primal norm $\|\cdot\|$, one must ensure that $\|\widehat{\theta}_t - \bar{\theta}_t\|_*$ remains small in order to obtain meaningful convergence guarantees when performing the OFTRL analysis. Since Lemma \ref{lem:concentration-theta} establishes convergence guarantees for $\max_{x \in \mathcal{X}} |\langle x, \widehat{\theta}_t - \theta_t \rangle|$, defining the dual norm as $\|z\|_* := \max_{x \in \mathcal{X}} |\langle x, z \rangle|$ is a natural choice.
\subsubsection{Suitable Regularizer}\label{sec:regularizer}
In this section, we formally state the regularizer $\phi(x)$ for the row player. Analogously, a corresponding regularizer $\psi(y)$ can be constructed for the column player. The objective is to design a regularizer, based on an ellipsoid approximating the action set $\mathcal{X}$, that is well-suited for the OFTRL framework by being $1$-strongly convex with respect to the primal norm and having polynomially bounded Bregman divergence.



Let $K:=\operatorname{conv}(\mathcal{X}\cup-\mathcal{X})$ and $\alpha:=\sqrt{d(d+1)}$.
One can compute an ellipsoid
\[
E=\{x:\ x^\top H x\le 1\},\quad H\succ0,
\]
such that
\[
E\ \subseteq\ K\ \subseteq\ \alpha\,E ,
\]
where $\alpha E:=\{\alpha x: x\in E\}$ (see Theorem 4.6.3 in \cite{grotschel2012geometric}).

Now we define our regularizer as 
\[
\phi(x)\;=\;\tfrac12\,x^\top H x,
\]
and the Bregman divergence with respect to the regularizer $\phi$ is defined as:
\[
{\displaystyle D_{\phi}(u,v)=\phi(u)-\phi(v)-\langle \nabla \phi(v),u-v\rangle .}
\]

We now establish the properties of the regularizer $\phi(\cdot)$. We begin by proving the strong convexity of $\phi(\cdot)$ in the following lemma.
\begin{lemma}[Strong-convexity]
    The regularizer $\phi(x)$ is $1$-strongly convex with respect to the primal norm $\|\cdot\|_\mathcal{X}$.
\end{lemma}
\begin{proof}

Let the polar of a set $C\subset \mathbb{R}^d$ be defined as $C^\circ:=\{z\in \mathbb{R}^d: \max_{x\in C}\langle x,z\rangle\leq 1 \}$. Observe that $A\subseteq B$ implies $B^\circ \subseteq A^\circ$ and $(\lambda A)^\circ=(1/\lambda)A^\circ$. Hence, we have $(1/\alpha)E^\circ\subseteq K^\circ\subseteq E^\circ$.

We now prove that $E^\circ=\{y:\,y^\top H^{-1}y\le 1\}$. Recall that $E=\{x:\|x\|_{H}\le 1\}$, where $\|x\|_{H}:=\sqrt{x^\top Hx}$.
By definition,
\[
E^\circ=\Bigl\{y:\ \sup_{\|x\|_{H}\le 1}\langle y,x\rangle\le 1\Bigr\}.
\]

Consider $x\in E$. Observe that $\langle y,x\rangle=\langle H^{-1/2}y,\,H^{1/2}x\rangle$ and $\|H^{1/2}x\|_2^2=x^\top H x\leq 1$. By Cauchy–Schwarz, we have
\[
\langle y,x\rangle\le \|H^{-1/2}y\|_2\,\|H^{1/2}x\|_2
\le \|H^{-1/2}y\|_2,
\]
with equality at $x=\frac{H^{-1}y}{\|H^{-1/2}y\|_2}$. Note that this choice of $x$ belongs to $E$ as $x^\top H x= \frac{y^\top H^{-1}y}{\|H^{-1/2}y\|_2^2}=1$.  Hence, we have
\[
\sup_{\|x\|_{H}\le 1}\langle y,x\rangle=\|H^{-1/2}y\|_2=\sqrt{y^\top H^{-1}y}.
\]

Therefore, $E^\circ=\{y:\sqrt{y^\top H^{-1}y}\le 1\}=\{y:\|y\|_{H^{-1}}\le 1\}$.

Using analogous calculations, we can also show that for any $z\in\mathbb{R}^d$, $\max_{y\in E^\circ}\langle y,z\rangle=\sqrt{z^\top H z}$.

Since $K^\circ\subseteq E^\circ$, we deduce the following:
\begin{align*}
    \|z\|_{\mathcal{X}}&=\max_{y:\max_{x\in\mathcal{X}}|\langle x,y\rangle|\leq 1}\langle y,z\rangle\\
    &=\max_{y:\max_{x\in K}\langle x,y\rangle\leq 1}\langle y,z\rangle\\
    &=\max_{y\in K^\circ}\langle y,z\rangle\\
    &\le \max_{y\in E^\circ}\langle y,z\rangle\\
    &=\sqrt{z^\top H z}=\|z\|_H
\end{align*}

For any $u,v\in\mathbb{R}^d$,
\[
D_\phi(u,v)=\tfrac12 (u-v)^\top H(u-v) \;=\; \tfrac12\|u-v\|_{H}^2.
\]
Since, $\|z\|_{H}\ge \|z\|$ for all $z\in \mathbb{R}^d$, we have
\[
D_\phi(u,v)\;\ge\;\tfrac12\|u-v\|^2,
\]
so $\phi$ is $1$-strongly convex with respect to $\|\cdot\|$. 

\end{proof}



Next, we show in the following proposition that the Bregman divergence is bounded by a polynomial in the dimension.
\begin{proposition}[Bregman divergence]\label{eq:bregman}
    For any $x,y\in\mathcal{X}$, we have $ D_\phi(x,y)\leq 2d(d+1)$.
\end{proposition}
\begin{proof}
For any $x,y\in \mathcal{X}\subseteq \alpha E$, we have $\|x-y\|_{H}\leq \|x\|_H+\|y\|_H\leq 2\alpha$. Therefore,
\begin{equation*}
    D_\phi(x,y)=\tfrac12\|x-y\|_{H}^2\ \le\ \tfrac12\left(2\alpha\right)^2\ =\ 2d(d+1).
\end{equation*}
\end{proof}
\subsection{Main result}\label{sec:main-result}
We now state below our main result, which is obtained by leveraging the properties of our estimator and regularizer, with technical details provided in the next section.
\begin{restatable}{theorem}{maintheorem}\label{thm:main_theorem}
Consider a two-player zero-sum game with action sets $\mathcal{X} \subseteq \mathbb{R}^n$ and $\mathcal{Y} \subseteq \mathbb{R}^m$, where both $\mathcal{X}$ and $\mathcal{Y}$ are convex and compact.  
Let $(x_k, y_k)$ denote the iterates generated by the two players running Algorithm~\ref{algo:last_iter_bandit_feedback} with step size $\eta = \tfrac{1}{6}$ in each round $k$. Then, with probability at least $1 - \delta$, for every $k \geq 1$, the iterate $(x_k,y_k)$ is an $\varepsilon_k$-approximate Nash equilibrium, where
\[
\varepsilon_k = \mathrm{poly}(n,m,\log(k/\delta)) \, k^{-1/4}.
\]
\end{restatable}


\section{TECHNICAL DETAILS FOR OUR MAIN RESULT}
The update step in each phase of Algorithm \ref{algo:last_iter_bandit_feedback} is an instance of OFTRL, a standard and widely used framework for online convex optimization. In each phase $t$, OFTRL outputs $\tilde{x}_t$ for the row player (and analogously $\tilde{y}_t$ for the column player). These outputs are then used to determine the actual strategies $x_k$ and $y_k$ chosen in each round of phase $t$. Let $u_t^x := A\tilde{y}_t$ denote the phase-$t$ utility vector of the row player, and $u_t^y := -A^\top \tilde{x}_t$ denote the phase-$t$ utility vector of the column player. Recall that $\bar{x}_T = \tfrac{1}{T}\sum_{t=1}^T \tilde{x}_t$ and $\bar{y}_T = \tfrac{1}{T}\sum_{t=1}^T \tilde{y}_t$. We have the following guarantee on the duality gap of $(\bar{x}_T,\bar{y}_T)$:
\begin{align*}
    &\quad\max_{x\in \mathcal{X}}\langle x,A\bar{y}_T\rangle-\min_{y\in\mathcal{Y}}\langle\bar{x}_T,Ay\rangle\\
    &=\frac{1}{T}\max_{x\in \mathcal{X}}\sum\limits_{t=1}^T\langle x,A\tilde{y}_t\rangle+\frac{1}{T}\max_{y\in\mathcal{Y}}\sum\limits_{t=1}^T\langle -A^\top \tilde{x}_t,y\rangle\\
    &=\frac{1}{T}\max_{x\in \mathcal{X}}\sum\limits_{t=1}^T\langle x-\tilde{x}_t,Ay_t\rangle\\
    &\quad+\frac{1}{T}\max_{y\in\mathcal{Y}}\sum\limits_{t=1}^T\langle -A^\top \tilde{x}_t,y-\tilde{y}_t\rangle\\
    &=\frac{1}{T}\left(\max_{x'\in \mathcal{X}}\sum_{t=1}^T\langle u_t^x,x'-\tilde{x}_t\rangle+\max_{y'\in \mathcal{Y}}\sum_{t=1}^T\langle u_t^y,y'-\tilde{y}_t\rangle\right)
\end{align*}

If the duality gap above is upper bounded by $\beta_T$, then $(\bar{x}_T,\bar{y}_T)$ is an $O(\beta_T)$-approximate Nash equilibrium. Since the iterates in phase $T$ take the form $((1-\lambda_T)\bar{x}_T+\lambda_T z_x,\,(1-\lambda_T)\bar{y}_T+\lambda_T z_y)$ with $z_x \in \mathcal{X}$ and $z_y \in \mathcal{Y}$, the iterates in phase $T$ are $O\!\left(\beta_T+\lambda_T\right)$-approximate Nash equilibria.  

Hence, our focus is on upper bounding the term $\max_{x'\in \mathcal{X}}\sum_{t=1}^T \langle u_t^x,\, x'-\tilde{x}_t\rangle$. To this end, we make use of the RVU property, which characterizes the performance of OFTRL. We state this property in the following lemma and include its proof in Appendix \ref{appsec:OFTRL_regularizer_proof} for completeness.

\begin{restatable}[RVU Property \cite{syrgkanis2015fast}]{lemma}{oftrl}\label{lemma:oftrl}
Let $\mathcal{X}\subset\mathbb{R}^d$ be compact convex.  
Let $R:\mathcal{X}\to\mathbb{R}$ be $\sigma$-strongly convex w.r.t.\ a norm $\|\cdot\|$ with dual $\|\cdot\|_*$.  
Fix a step size $\eta>0$ and initialize $\tilde{x}_0=\tilde{x}_1=\arg\min_{x\in\mathcal{X}} R(x)$. Assume $u_0 = 0$.

For utilities $u_t\in\mathbb{R}^d$, define the OFTRL decisions
\[
\tilde{x}_{t+1} \in \arg \max _{x \in \mathcal{X}}\left\{\left\langle x, \sum_{\ell=1}^{t} u_\ell+u_{t}\right\rangle-\frac{1}{\eta} R(x)\right\}.
\]

Then for every $x\in \mathcal{X}$,
\begin{align*}
 \sum_{t=1}^T\langle u_t,\,x-\tilde{x}_t\rangle
&\le \frac{D_{R}(x,\tilde{x}_1)}{\eta}
+ \frac{\eta}{\sigma}\sum_{t=1}^T\|u_t-u_{t-1}\|_*^2\\
&- \frac{\sigma}{4\eta}\sum_{t=1}^T\|\tilde{x}_t-\tilde{x}_{t-1}\|^2. 
\end{align*}
\end{restatable}

While the RVU property is sufficient to bound 
$\max_{x' \in \mathcal{X}} \sum_{t=1}^T \langle u_t^x, x' - \tilde{x}_t \rangle$ 
when the actual phase-wise utilities are available, this is not the case in our algorithm. Instead, we use the estimates $\widehat{u}_t^x$ in the row player's OFTRL updates, since we operate under the bandit feedback setting. 

To address this challenge, we rely on the following key observation concerning the actual phase-wise utilities:
\begin{align*}
    u_t^x = A\tilde{y}_t &= \sum_{\ell=1}^t A\tilde{y}_{\ell} - \sum_{\ell=1}^{t-1} A\tilde{y}_{\ell} \\
          &= tA\bar{y}_t - (t-1)A\bar{y}_{t-1},
\end{align*}
where $\bar{y}_t = \tfrac{1}{t}\sum_{s=1}^t \tilde{y}_s$. Also note that the estimate $\widehat{u}_t^x$ has the corresponding form 
$t \cdot \widehat{\theta}_t^x - (t-1) \cdot \widehat{\theta}_{t-1}^x$. This allows us to upper bound $\langle u_t^x, x - \tilde{x}_t \rangle$ as
\begin{align*}
     \langle \widehat{u}_t^x, x - \tilde{x}_t \rangle 
     - \langle \widehat{u}_t^x - \mathbb{E}[\widehat{u}_t^x], x - \tilde{x}_t \rangle 
     + O\!\left(\tfrac{1}{t}\right).
\end{align*}

Moreover, one can show that
\begin{align*}
    \|\mathbb{E}[\widehat{u}_t^x] - \mathbb{E}[\widehat{u}_{t-1}^x]\|_{*,\mathcal{X}} 
     \;\leq\; \|u_t^x - u_{t-1}^x\|_{*,\mathcal{X}}
     + O\!\left(\tfrac{1}{t}\right).
\end{align*}

These observations allow us to apply the RVU property to the sequence of estimates $\widehat{u}_t^x$, yielding the following lemma. All missing proofs, including that of the lemma, are deferred to Appendix~\ref{appendix:regret-estimation-error}.
\begin{restatable}[RVU with estimation error]{lemma}{regretwithestimationerror}\label{lemma:regret_with_error}
Let $\Delta_t^x := \widehat{\theta}_t^x - \bar{\theta}_t^x$. Then for any $x \in \cX$ and any $\eta > 0$, we have the following for the row player:
\begin{align*}
&\sum_{t=1}^T \langle u_t^x, x-\tilde{x}_t \rangle \\
&\le \frac{D_\phi(x,\tilde{x}_1)}{\eta} + 2\|T\Delta_T^x\|_{*,\mathcal{X}} \\
&\quad + 36\eta\sum_{t=1}^{T}\|t \Delta_t^x\|_{*,\mathcal{X}}^2 + 4\eta\sum_{t=1}^T \|u_t-u_{t-1}\|_{*,\mathcal{X}}^2 \\
&\quad - \frac{3}{16\eta}\sum_{t=1}^T\|\tilde{x}_t-\tilde{x}_{t-1}\|^2_\mathcal{X} + O\!\left(\eta+\log T\right).
\end{align*}
\end{restatable}
Analogously, we can provide a similar guarantee for the column player, with $\Delta_t^y$ defined accordingly. We are now ready to upper bound the duality gap of $(\bar{x}_T,\bar{y}_T)$ using Lemma \ref{lemma:regret_with_error}. First, we establish the following:
\begin{align*}
    &\max_{x'\in \mathcal{X}}\sum_{t=1}^T\langle u_t^x,x'-\tilde{x}_t\rangle+\max_{y'\in \mathcal{Y}}\sum_{t=1}^T\langle u_t^y,y'-\tilde{y}_t\rangle \\
    &\le \frac{D_{\phi}(x,\tilde{x}_1)}{\eta} + 2\|T\Delta^x_T\|_{*,\mathcal{X}} \\
&+ {36\eta}\sum_{t=1}^{T}\|t \Delta^x_t\|_{*,\mathcal{X}}^2 + {4\eta} \sum_{t=1}^T \|u^x_t-u^x_{t-1}\|_{*,\mathcal{X}}^2\\
&- \frac{3}{16\eta}\sum_{t=1}^T\|\tilde{x}_t-\tilde{x}_{t-1}\|^2_{\mathcal{X}} \\
&+\frac{D_{\psi}(y,\tilde{y}_1)}{\eta} + 2\|T\Delta^y_T\|_{*,\mathcal{Y}} \\
&+ {36\eta}\sum_{t=1}^{T}\|t \Delta^y_t\|_{*,\mathcal{Y}}^2 + {4\eta}\sum_{t=1}^T \|u^y_t-u^y_{t-1}\|_{*,\mathcal{Y}}^2\\
&- \frac{3}{16\eta}\sum_{t=1}^T\|\tilde{y}_t-\tilde{y}_{t-1}\|^2_{\mathcal{Y}}+O(\eta+\log T)
\end{align*}

We group terms on the right side. We begin with:
\begin{align*}
    \text{Term I} &= {4\eta} \sum_{t=1}^T (\|u^x_t-u^x_{t-1}\|_{*,\mathcal{X}}^2 +  \|u^y_t-u^y_{t-1}\|_{*,\mathcal{Y}}^2) \\
&\quad- \frac{3}{16\eta}\sum_{t=1}^T(\|\tilde{x}_t-\tilde{x}_{t-1}\|^2_{\mathcal{X}}  + \|\tilde{y}_t-\tilde{y}_{t-1}\|^2_{\mathcal{Y}})
\end{align*}
Now observe that 
\begin{align*}
    \|u^x_t-u^x_{t-1}\|_{*,\mathcal{X}}&=\sup_{x\in\mathcal{X}}|\langle x,A(\tilde{y}_t-\tilde{y}_{t-1})|\\
    &\leq \sup_{x\in\mathcal{X}}\|A^\top x\|_{*,\mathcal{Y}}\|\tilde{y}_t-\tilde{y}_{t-1}\|_{\mathcal{Y}}\\
    &\leq \|\tilde{y}_t-\tilde{y}_{t-1}\|_{\mathcal{Y}},
\end{align*}
where the first inequality follows from the fact that $|\langle x,y\rangle|\leq \|x\|\|y\|_*$ for any primal-dual norm pair and the second inequality follows from the fact that $|\langle x,Ay\rangle|\leq 1$ for all $x\in\mathcal{X}$ and $y\in\mathcal{Y}$. Similarly we can show that $\|u^y_t-u^y_{t-1}\|_{*,\mathcal{Y}}\leq \|\tilde{x}_t-\tilde{x}_{t-1}\|_{\mathcal{X}}$. As $\eta= 1/6$, we have $\text{Term I}\leq 0$.


Let $\text{Term II}=2\|T\Delta^x_T\|_{*,\mathcal{X}}+ {36\eta}\sum_{t=1}^{T}\|t \Delta^x_t\|_{*,\mathcal{X}}^2 + 2\|T\Delta^y_T\|_{*,\mathcal{Y}}+ {36\eta}\sum_{t=1}^{T}\|t \Delta^y_t\|_{*,\mathcal{Y}}^2 $.

Due to Lemma \ref{lem:concentration-theta}, we get that $\|\Delta^x_t\|_{*,\mathcal{X}}\leq 48\sqrt{\frac{n^3}{t^3}}$ with probability at least $1-\frac{\delta}{4t^2}$. Analogously, we can show that $\|\Delta^y_t\|_{*,\mathcal{Y}}\leq 48\sqrt{\frac{m^3}{t^3}}$ with probability at least $1-\frac{\delta}{4t^2}$. Hence due to union bound, if $\eta=1/6$, we get the following with probability at least $1-\delta$:
\begin{align*}
    \text{Term II}&\leq O\left(\sqrt{\frac{n^3}{T}}+\sqrt{\frac{m^3}{T}}+\sum_{t=1}^T\frac{t^2(n^3+m^3)}{t^3}\right)\\
    &\leq O((n^3+m^3)\log T)
\end{align*}

Define $\text{Term III}=\frac{D_{\phi}(x,\tilde{x}_1)}{\eta}+\frac{D_{\psi}(y,\tilde{y}_1)}{\eta}+O(\eta+\log T)$.

Due to Proposition \ref{eq:bregman}, we have $D_{\phi}(x,\tilde{x}_1)\leq O(n^2)$. Similarly, we can show that $D_{\psi}(y,\tilde{y}_1)\leq O(m^2)$. Hence, $\text{Term III}\leq O(n^2+m^2+\log T)$ if $\eta=1/6$.

Hence, we have the following:
\begin{align*}
     &\max_{x'\in \mathcal{X}}\sum_{t=1}^T\langle u_t^x,x'-\tilde{x}_t\rangle+\max_{y'\in \mathcal{Y}}\sum_{t=1}^T\langle u_t^y,y'-\tilde{y}_t\rangle\\
     &\leq \text{Term I}+\text{Term II}+\text{Term III}\\
     &\leq O((n^3+m^3)\log T)
\end{align*}

Hence $(\bar{x}_T,\bar{y}_T)$ is an $O\left(\frac{(n^3+m^3)\log T}{T}\right)$-approximate Nash equilibrium. Hence, the pair of strategies played in a round in the phase $T$ is  $O\left(\frac{(n^3+m^3)\log T}{T}+\lambda_T\right)$-approximate Nash equilibrium which is also $O\left(\frac{(n^3+m^3)\log T}{T}\right)$-approximate Nash equilibrium. Now consider a round $k$ and let it be part of phase $T_k$. Note that $T_k\leq k$. As in each phase $t$ we have $B_t=\log(8t^2/\delta)\cdot t^3$ rounds, we therefore have $\log(8k^2/\delta)\cdot T_k^4\geq k$. Hence, we have $T_k\geq \left(\frac{k}{\log(8k^2/\delta)}\right)^{1/4}$. Hence, the iterate $(x_k,y_k)$ is an $\varepsilon_k$-approximate Nash equilibrium, where $\varepsilon_k$ is upper bounded as follows:
\begin{equation*}
    \varepsilon_k\leq O((n^3+m^3)\log (k)\log^{1/4}(8k^2/\delta)k^{-1/4})
\end{equation*}

\section{COMPUTATIONAL EFFICIENCY OF OUR ALGORITHM}\label{sec:comp}
In this section, we show that our proposed algorithm is computationally efficient, provided the action sets admit an efficient linear optimization oracle. We establish this by showing that its key building blocks can be implemented in polynomial time.  

First, recall that our exploration distribution $\mathcal{D}_{\mathcal{X}}$ is uniform over a subset $S := \{x_1, \dots, x_n\} \subset \mathcal{X}$ such that
$$
\sup_{x \in \mathcal{X}} x^\top V^{-1} x \;\le\; 2 n^2,
\qquad
V := \frac{1}{n}\sum_{i=1}^n x_i x_i^\top \ \succ 0.
$$
\cite{hazan2016volumetric} showed that such a subset can be computed in polynomial time, provided $\mathcal{X}$ admits an efficient linear optimization oracle. Hence, our sampling process is efficient.  

Next, recall that we construct an ellipsoid $E$ such that $E \subseteq K \subseteq \sqrt{n(n+1)}E$, where $K = \mathrm{conv}(\mathcal{X} \cup -\mathcal{X})$. Theorem 4.6.3 of \cite{grotschel2012geometric} shows that such an ellipsoid can be computed in polynomial time, provided there is an efficient linear optimization oracle for $K$. For any $z \in \mathbb{R}^n$, we have 
\[
\max_{x \in K} \langle x, z \rangle 
= \max \big\{ \max_{x \in \mathcal{X}} \langle x, z \rangle, \; \max_{x \in \mathcal{X}} \langle -x, z \rangle \big\}.
\]
Hence, $K$ admits an efficient linear optimization oracle whenever $\mathcal{X}$ does.  

Finally, our OFTRL update step is a convex optimization problem over an action set that admits efficient linear optimization. Such an OFTRL update can be implemented in polynomial time, provided we can compute the regularizer and its gradient efficiently—which we can in the case of our regularizer (see Chapter 2 of \cite{grotschel2012geometric}).  

These are the three main components of our algorithm. Therefore, the algorithm can be implemented in polynomial time, provided the action sets admit an efficient linear optimization oracle.

\section{CONCLUSION}
In this paper, we presented the first uncoupled learning dynamics whose iterates exhibit last-iterate convergence with high probability under bandit feedback for bilinear saddle-point problems over convex sets. We established a convergence rate of $\tilde{O}(T^{-1/4})$ and showed that our dynamics can be implemented efficiently, provided the action sets admit efficient linear optimization oracles. This work raises several interesting open questions.  

First, what is the tight lower bound on the last-iterate convergence rate for bilinear saddle-point problems under bandit feedback, and can we design uncoupled learning dynamics that achieve this rate? Next, does there exist a simpler dynamics that applies optimistic FTRL in each round and attains last-iterate convergence under bandit feedback, rather than relying on phased updates and the involved sampling procedure used in our algorithm? Finally, can these results be generalized to convex-concave functions and monotone games, and can last-iterate convergence be achieved in these broader settings under bandit feedback?

\section*{ACKNOWLEDGEMENTS}

Ioannis Panageas was supported by National Science Foundation grant CCF-2454115.
LJ Ratliff was supported in part by NSF 1844729, 2312775.
KJ and AM were supported in part by NSF 2141511, 2023239, and a Singapore AI Visiting Professorship award. 
JM and CJZ were supported in part by NSF ID 2045402 and a  Simons Collaboration on the Theory of Algorithmic Fairness.
\bibliography{biblio}

\begin{thebibliography}{43}
\providecommand{\natexlab}[1]{#1}
\providecommand{\url}[1]{\texttt{#1}}
\expandafter\ifx\csname urlstyle\endcsname\relax
  \providecommand{\doi}[1]{doi: #1}\else
  \providecommand{\doi}{doi: \begingroup \urlstyle{rm}\Url}\fi

\bibitem[Abe et~al.(2024)Abe, Sakamoto, Ariu, and Iwasaki]{abe2024boosting}
Kenshi Abe, Mitsuki Sakamoto, Kaito Ariu, and Atsushi Iwasaki.
\newblock Boosting perturbed gradient ascent for last-iterate convergence in
  games.
\newblock \emph{arXiv preprint arXiv:2410.02388}, 2024.

\bibitem[Anagnostides et~al.(2022{\natexlab{a}})Anagnostides, Farina, Kroer,
  Lee, Luo, and Sandholm]{anagnostides2022uncoupled}
Ioannis Anagnostides, Gabriele Farina, Christian Kroer, Chung-Wei Lee, Haipeng
  Luo, and Tuomas Sandholm.
\newblock Uncoupled learning dynamics with $o(\log t) $ swap regret in
  multiplayer games.
\newblock \emph{Advances in Neural Information Processing Systems},
  35:\penalty0 3292--3304, 2022{\natexlab{a}}.

\bibitem[Anagnostides et~al.(2022{\natexlab{b}})Anagnostides, Panageas, Farina,
  and Sandholm]{anagnostides2022last}
Ioannis Anagnostides, Ioannis Panageas, Gabriele Farina, and Tuomas Sandholm.
\newblock On last-iterate convergence beyond zero-sum games.
\newblock In \emph{International Conference on Machine Learning}, pages
  536--581. PMLR, 2022{\natexlab{b}}.

\bibitem[Auer et~al.(2002)Auer, Cesa-Bianchi, Freund, and
  Schapire]{auer2002nonstochastic}
Peter Auer, Nicolo Cesa-Bianchi, Yoav Freund, and Robert~E Schapire.
\newblock The nonstochastic multiarmed bandit problem.
\newblock \emph{SIAM journal on computing}, 32\penalty0 (1):\penalty0 48--77,
  2002.

\bibitem[Ba et~al.(2025)Ba, Lin, Zhang, and Zhou]{ba2025doubly}
Wenjia Ba, Tianyi Lin, Jiawei Zhang, and Zhengyuan Zhou.
\newblock Doubly optimal no-regret online learning in strongly monotone games
  with bandit feedback.
\newblock \emph{Operations Research}, 2025.

\bibitem[Bailey and Piliouras(2018)]{bailey2018multiplicative}
James~P Bailey and Georgios Piliouras.
\newblock Multiplicative weights update in zero-sum games.
\newblock In \emph{Proceedings of the 2018 ACM Conference on Economics and
  Computation}, pages 321--338, 2018.

\bibitem[Besbes and Zeevi(2009)]{besbes2009dynamic}
Omar Besbes and Assaf Zeevi.
\newblock Dynamic pricing without knowing the demand function: Risk bounds and
  near-optimal algorithms.
\newblock \emph{Operations research}, 57\penalty0 (6):\penalty0 1407--1420,
  2009.

\bibitem[Brown and Sandholm(2018)]{brown2018superhuman}
Noam Brown and Tuomas Sandholm.
\newblock Superhuman ai for heads-up no-limit poker: Libratus beats top
  professionals.
\newblock \emph{Science}, 359\penalty0 (6374):\penalty0 418--424, 2018.

\bibitem[Bubeck et~al.(2012)Bubeck, Cesa-Bianchi, and
  Kakade]{bubeck2012towards}
S{\'e}bastien Bubeck, Nicolo Cesa-Bianchi, and Sham~M Kakade.
\newblock Towards minimax policies for online linear optimization with bandit
  feedback.
\newblock In \emph{Conference on Learning Theory}, pages 41--1. JMLR Workshop
  and Conference Proceedings, 2012.

\bibitem[Cai et~al.(2022)Cai, Oikonomou, and Zheng]{cai2022finite}
Yang Cai, Argyris Oikonomou, and Weiqiang Zheng.
\newblock Finite-time last-iterate convergence for learning in multi-player
  games.
\newblock \emph{Advances in Neural Information Processing Systems},
  35:\penalty0 33904--33919, 2022.

\bibitem[Cai et~al.(2023)Cai, Luo, Wei, and Zheng]{cai2023uncoupled}
Yang Cai, Haipeng Luo, Chen-Yu Wei, and Weiqiang Zheng.
\newblock Uncoupled and convergent learning in two-player zero-sum markov games
  with bandit feedback.
\newblock \emph{Advances in Neural Information Processing Systems},
  36:\penalty0 36364--36406, 2023.

\bibitem[Cai et~al.(2024)Cai, Farina, Grand-Cl{\'e}ment, Kroer, Lee, Luo, and
  Zheng]{cai2024fast}
Yang Cai, Gabriele Farina, Julien Grand-Cl{\'e}ment, Christian Kroer, Chung-Wei
  Lee, Haipeng Luo, and Weiqiang Zheng.
\newblock Fast last-iterate convergence of learning in games requires forgetful
  algorithms.
\newblock \emph{Advances in Neural Information Processing Systems},
  37:\penalty0 23406--23434, 2024.

\bibitem[Cai et~al.(2025)Cai, Luo, Wei, and Zheng]{cai2025average}
Yang Cai, Haipeng Luo, Chen-Yu Wei, and Weiqiang Zheng.
\newblock From average-iterate to last-iterate convergence in games: A
  reduction and its applications.
\newblock \emph{arXiv preprint arXiv:2506.03464, To appear at NeurIPS}, 2025.

\bibitem[Chandrasekaran et~al.(2012)Chandrasekaran, Recht, Parrilo, and
  Willsky]{chandrasekaran2012convex}
Venkat Chandrasekaran, Benjamin Recht, Pablo~A Parrilo, and Alan~S Willsky.
\newblock The convex geometry of linear inverse problems.
\newblock \emph{Foundations of Computational mathematics}, 12\penalty0
  (6):\penalty0 805--849, 2012.

\bibitem[Chen and Peng(2020)]{chen2020hedging}
Xi~Chen and Binghui Peng.
\newblock Hedging in games: Faster convergence of external and swap regrets.
\newblock \emph{Advances in Neural Information Processing Systems},
  33:\penalty0 18990--18999, 2020.

\bibitem[Chen et~al.(2023)Chen, Zhang, Mazumdar, Ozdaglar, and
  Wierman]{chen2023finite}
Zaiwei Chen, Kaiqing Zhang, Eric Mazumdar, Asuman Ozdaglar, and Adam Wierman.
\newblock A finite-sample analysis of payoff-based independent learning in
  zero-sum stochastic games.
\newblock \emph{Advances in Neural Information Processing Systems},
  36:\penalty0 75826--75883, 2023.

\bibitem[Chen et~al.(2024)Chen, Zhang, Mazumdar, Ozdaglar, and
  Wierman]{chen2024last}
Zaiwei Chen, Kaiqing Zhang, Eric Mazumdar, Asuman Ozdaglar, and Adam Wierman.
\newblock Last-iterate convergence of payoff-based independent learning in
  zero-sum stochastic games.
\newblock \emph{arXiv preprint arXiv:2409.01447}, 2024.

\bibitem[Daskalakis and Panageas(2018)]{daskalakis2018last}
Constantinos Daskalakis and Ioannis Panageas.
\newblock Last-iterate convergence: Zero-sum games and constrained min-max
  optimization.
\newblock \emph{arXiv preprint arXiv:1807.04252}, 2018.

\bibitem[Daskalakis et~al.(2011)Daskalakis, Deckelbaum, and
  Kim]{daskalakis2011near}
Constantinos Daskalakis, Alan Deckelbaum, and Anthony Kim.
\newblock Near-optimal no-regret algorithms for zero-sum games.
\newblock In \emph{Proceedings of the twenty-second annual ACM-SIAM symposium
  on Discrete Algorithms}, pages 235--254. SIAM, 2011.

\bibitem[Daskalakis et~al.(2017)Daskalakis, Ilyas, Syrgkanis, and
  Zeng]{daskalakis2017training}
Constantinos Daskalakis, Andrew Ilyas, Vasilis Syrgkanis, and Haoyang Zeng.
\newblock Training gans with optimism.
\newblock \emph{arXiv preprint arXiv:1711.00141}, 2017.

\bibitem[Daskalakis et~al.(2021)Daskalakis, Fishelson, and
  Golowich]{daskalakis2021near}
Constantinos Daskalakis, Maxwell Fishelson, and Noah Golowich.
\newblock Near-optimal no-regret learning in general games.
\newblock \emph{Advances in Neural Information Processing Systems},
  34:\penalty0 27604--27616, 2021.

\bibitem[Den~Boer(2015)]{den2015dynamic}
Arnoud~V Den~Boer.
\newblock Dynamic pricing and learning: Historical origins, current research,
  and new directions.
\newblock \emph{Surveys in operations research and management science},
  20\penalty0 (1):\penalty0 1--18, 2015.

\bibitem[Dong et~al.(2024)Dong, Wang, and Yu]{dong2024uncoupled}
Jing Dong, Baoxiang Wang, and Yaoliang Yu.
\newblock Uncoupled and convergent learning in monotone games under bandit
  feedback.
\newblock \emph{arXiv preprint arXiv:2408.08395}, 2024.

\bibitem[Fiegel et~al.(2025)Fiegel, Menard, Kozuno, Valko, and
  Perchet]{fiegel2025harder}
C{\^o}me Fiegel, Pierre Menard, Tadashi Kozuno, Michal Valko, and Vianney
  Perchet.
\newblock The harder path: Last iterate convergence for uncoupled learning in
  zero-sum games with bandit feedback.
\newblock In \emph{42nd International Conference on Machine Learning (ICML
  2025)}, volume 267, 2025.

\bibitem[Giannou et~al.(2021)Giannou, Vlatakis-Gkaragkounis, and
  Mertikopoulos]{giannou2021rate}
Angeliki Giannou, Emmanouil-Vasileios Vlatakis-Gkaragkounis, and Panayotis
  Mertikopoulos.
\newblock On the rate of convergence of regularized learning in games: From
  bandits and uncertainty to optimism and beyond.
\newblock \emph{Advances in Neural Information Processing Systems},
  34:\penalty0 22655--22666, 2021.

\bibitem[Gr{\"o}tschel et~al.(2012)Gr{\"o}tschel, Lov{\'a}sz, and
  Schrijver]{grotschel2012geometric}
Martin Gr{\"o}tschel, L{\'a}szl{\'o} Lov{\'a}sz, and Alexander Schrijver.
\newblock \emph{Geometric algorithms and combinatorial optimization}, volume~2.
\newblock Springer Science \& Business Media, 2012.

\bibitem[Hazan and Karnin(2016)]{hazan2016volumetric}
Elad Hazan and Zohar Karnin.
\newblock Volumetric spanners: an efficient exploration basis for learning.
\newblock \emph{The Journal of Machine Learning Research}, 17\penalty0
  (1):\penalty0 4062--4095, 2016.

\bibitem[Jordan et~al.(2025)Jordan, Lin, and Zhou]{jordan2025adaptive}
Michael Jordan, Tianyi Lin, and Zhengyuan Zhou.
\newblock Adaptive, doubly optimal no-regret learning in strongly monotone and
  exp-concave games with gradient feedback.
\newblock \emph{Operations Research}, 73\penalty0 (3):\penalty0 1675--1702,
  2025.

\bibitem[Krichene et~al.(2015)Krichene, Drigh{\`e}s, and
  Bayen]{krichene2015online}
Walid Krichene, Benjamin Drigh{\`e}s, and Alexandre~M Bayen.
\newblock Online learning of nash equilibria in congestion games.
\newblock \emph{SIAM Journal on Control and Optimization}, 53\penalty0
  (2):\penalty0 1056--1081, 2015.

\bibitem[Lattimore and Szepesv{\'a}ri(2020)]{lattimore2020bandit}
Tor Lattimore and Csaba Szepesv{\'a}ri.
\newblock \emph{Bandit algorithms}.
\newblock Cambridge University Press, 2020.

\bibitem[Liang and Stokes(2019)]{liang2019interaction}
Tengyuan Liang and James Stokes.
\newblock Interaction matters: A note on non-asymptotic local convergence of
  generative adversarial networks.
\newblock In \emph{The 22nd International Conference on Artificial Intelligence
  and Statistics}, pages 907--915. PMLR, 2019.

\bibitem[Luo()]{haipeng_lecture_notes}
Haipeng Luo.
\newblock Lecture notes: Introduction to online optimization/learning.
\newblock URL
  \url{https://haipeng-luo.net/courses/CSCI659/2022_fall/lectures/lecture3.pdf}.

\bibitem[Mertikopoulos et~al.(2018)Mertikopoulos, Papadimitriou, and
  Piliouras]{mertikopoulos2018cycles}
Panayotis Mertikopoulos, Christos Papadimitriou, and Georgios Piliouras.
\newblock Cycles in adversarial regularized learning.
\newblock In \emph{Proceedings of the twenty-ninth annual ACM-SIAM symposium on
  discrete algorithms}, pages 2703--2717. SIAM, 2018.

\bibitem[{Meta Fundamental AI Research Diplomacy Team (FAIR)} et~al.(2022){Meta
  Fundamental AI Research Diplomacy Team (FAIR)}, Bakhtin, Brown, Dinan,
  Farina, Flaherty, Fried, Goff, Gray, Hu, et~al.]{meta2022human}
{Meta Fundamental AI Research Diplomacy Team (FAIR)}, Anton Bakhtin, Noam
  Brown, Emily Dinan, Gabriele Farina, Colin Flaherty, Daniel Fried, Andrew
  Goff, Jonathan Gray, Hengyuan Hu, et~al.
\newblock Human-level play in the game of diplomacy by combining language
  models with strategic reasoning.
\newblock \emph{Science}, 378\penalty0 (6624):\penalty0 1067--1074, 2022.

\bibitem[Munos et~al.(2023)Munos, Valko, Calandriello, Azar, Rowland, Guo,
  Tang, Geist, Mesnard, Michi, et~al.]{munos2023nash}
R{\'e}mi Munos, Michal Valko, Daniele Calandriello, Mohammad~Gheshlaghi Azar,
  Mark Rowland, Zhaohan~Daniel Guo, Yunhao Tang, Matthieu Geist, Thomas
  Mesnard, Andrea Michi, et~al.
\newblock Nash learning from human feedback.
\newblock \emph{arXiv preprint arXiv:2312.00886}, 18, 2023.

\bibitem[Muthukumar et~al.(2020)Muthukumar, Phade, and
  Sahai]{muthukumar2020impossibility}
Vidya Muthukumar, Soham Phade, and Anant Sahai.
\newblock On the impossibility of convergence of mixed strategies with no
  regret learning.
\newblock \emph{arXiv preprint arXiv:2012.02125}, 2020.

\bibitem[Neu(2015)]{neu2015explore}
Gergely Neu.
\newblock Explore no more: Improved high-probability regret bounds for
  non-stochastic bandits.
\newblock \emph{Advances in Neural Information Processing Systems}, 28, 2015.

\bibitem[Rakhlin and Sridharan(2013)]{rakhlin2013online}
Alexander Rakhlin and Karthik Sridharan.
\newblock Online learning with predictable sequences.
\newblock In \emph{Conference on Learning Theory}, pages 993--1019. PMLR, 2013.

\bibitem[Silver et~al.(2017)Silver, Schrittwieser, Simonyan, Antonoglou, Huang,
  Guez, Hubert, Baker, Lai, Bolton, et~al.]{silver2017mastering}
David Silver, Julian Schrittwieser, Karen Simonyan, Ioannis Antonoglou, Aja
  Huang, Arthur Guez, Thomas Hubert, Lucas Baker, Matthew Lai, Adrian Bolton,
  et~al.
\newblock Mastering the game of go without human knowledge.
\newblock \emph{nature}, 550\penalty0 (7676):\penalty0 354--359, 2017.

\bibitem[Syrgkanis et~al.(2015)Syrgkanis, Agarwal, Luo, and
  Schapire]{syrgkanis2015fast}
Vasilis Syrgkanis, Alekh Agarwal, Haipeng Luo, and Robert~E Schapire.
\newblock Fast convergence of regularized learning in games.
\newblock \emph{Advances in Neural Information Processing Systems}, 28, 2015.

\bibitem[Tatarenko and Kamgarpour(2019)]{tatarenko2019learning}
Tatiana Tatarenko and Maryam Kamgarpour.
\newblock Learning nash equilibria in monotone games.
\newblock In \emph{2019 IEEE 58th Conference on Decision and Control (CDC)},
  pages 3104--3109. IEEE, 2019.

\bibitem[Wei et~al.(2020)Wei, Lee, Zhang, and Luo]{wei2020linear}
Chen-Yu Wei, Chung-Wei Lee, Mengxiao Zhang, and Haipeng Luo.
\newblock Linear last-iterate convergence in constrained saddle-point
  optimization.
\newblock \emph{arXiv preprint arXiv:2006.09517}, 2020.

\bibitem[Zimmert and Lattimore(2022)]{zimmert2022return}
Julian Zimmert and Tor Lattimore.
\newblock Return of the bias: Almost minimax optimal high probability bounds
  for adversarial linear bandits.
\newblock In \emph{Conference on Learning Theory}, pages 3285--3312. PMLR,
  2022.

\end{thebibliography}
\bibliographystyle{plainnat}

\clearpage
\appendix
\pagenumbering{gobble} 
\thispagestyle{empty}

\onecolumn
\aistatstitle{
Supplementary Material}

In this appendix, we provide the missing details from the main body. In Appendix \ref{appendix:estimation}, we present the formal guarantees for our estimators along with their proofs. In Appendix \ref{appendix:norms}, we provide the missing proofs for our primal–dual norms. In Appendix \ref{appsec:OFTRL_regularizer_proof}, we give the proof of the RVU property of OFTRL for completeness. In Appendix \ref{appendix:regret-estimation-error}, we establish the RVU property under estimation error. 
\section{Estimation}\label{appendix:estimation}
\begin{lemma}\label{subgaussian-noise-lemma}
Fix $x \in \mathcal{X}$ and $\lambda \in (0,1)$. Let $\mathcal{D}$ be a distribution over $\mathcal{Y}$.  
Define the random variable $\tilde{y}$ as follows: with probability $1/2$, set $\tilde{y} = \bar{y}$ for some fixed $\bar{y} \in \mathcal{Y}$; with probability $1/2$, sample $z \sim \mathcal{D}$ and set $\tilde{y} = (1-\lambda)\bar{y} + \lambda z$.  
If $\widehat{y} = \mathbb{E}[\tilde{y}]$, then $\langle x, A(\tilde{y} - \widehat{y}) \rangle$ is zero-mean and $4\lambda^{2}$-subgaussian.
\end{lemma}

\begin{proof}
By linearity of expectation, $\langle x, A(\tilde y - \widehat y)\rangle$ has mean zero. Note that $\tilde y$ is always of the form $(1-\lambda)\bar{y}+\lambda z'$ where $z'\in \mathcal{Y}$. Next, observe that $\widehat y=(1-\lambda)\bar{y}+\lambda \widehat z\in \mathcal{Y}$, where $\widehat z= \frac{1}{2}\bar{y}+\frac{1}{2}\mathbb{E}_{z\sim \mathcal{D}}[z]\in \mathcal{Y}$. 
Since $\langle x, Ay \rangle \in [-1,1]$ for all $x \in \mathcal{X}$ and $y \in \mathcal{Y}$, it follows that  
\[
\langle x, A(\tilde y - \widehat y)\rangle \in [-2\lambda,2\lambda].
\]  
The result then follows from the fact that a bounded zero-mean random variable taking values in an interval of length $4\lambda$ is $4\lambda^2$-subgaussian.  
\end{proof}

\begin{lemma}[Chernoff Bound]
    Let $X_1,X_2,\ldots,X_n$ be i.i.d samples  from a Bernoulli distribution with mean $\mu$. Then we have the following for any $0<\delta<1$:
\begin{equation*}
    \mathbb{P}\left[\frac{1}{n}\cdot\sum_{i=1}^nX_i\geq (1+\delta) \mu\right]\leq  e^{-\frac{n\mu\delta^2}{3}} \quad\text{and}\quad \mathbb{P}\left[\frac{1}{n}\cdot\sum_{i=1}^nX_i\leq (1-\delta) \mu\right]\leq  e^{-\frac{n\mu\delta^2}{2}} 
\end{equation*}
\end{lemma}

\subsection{Estimates using exploration distribution}\label{appendix-sec:estimate-exp-dist}
Let $\mathcal{X}\subset\mathbb{R}^n$ be convex, compact, and $\mathrm{span}(\mathcal{X})=\mathbb{R}^n$.
Consider a subset  $S:=\{x_1,\dots,x_n\}\subset\mathcal{X}$ such that

$$
\sup_{x\in\mathcal{X}}\, x^\top V^{-1}x \;\le\; 2 n^2,
\qquad
V \;:=\; \frac{1}{n}\sum_{i=1}^n x_i x_i^\top \ \succ 0.
$$
Such a subset can be computed in polynomial time, provided $\mathcal{X}$ has an efficient linear optimization oracle (see \cite{hazan2016volumetric}). Now collect $N$ samples by repeating each $x_i$ exactly $r:=N/n$ times (assume $n$ divides $N$). The observations follow

$$
y_t \;=\; \langle x_t,\theta\rangle + \eta_t,
\qquad
t=1,\dots,N,
$$

where $\theta\in\mathbb{R}^n$ is fixed and $\{\eta_t\}_{t\in [N]}$ are independent, mean-zero, $\sigma^2$-subgaussian (MGF sense).

Define the matrix

$$
V \;=\; \frac{1}{N}\sum_{t=1}^N x_t x_t^\top \;=\; \frac{1}{n}\sum_{i=1}^n x_i x_i^\top,
$$

and the ordinary least squares estimator

$$
\widehat\theta \;=\; V^{-1}\left(\frac{1}{N}\sum_{t=1}^Nx_ty_t\right) .
$$

As $y_t \;=\; \langle x_t,\theta\rangle + \eta_t$, we get the following:

$$
\widehat\theta-\theta
\;=\;
V^{-1}\Big(\frac{1}{N}\sum_{t=1}^N x_t\,\eta_t\Big).
$$

Let $
Z \;:=\; V^{1/2}(\widehat\theta-\theta)
\;=\; \frac{1}{N}\sum_{t=1}^N V^{-1/2}x_t\,\eta_t .
$

For any $z\in\mathbb{R}^n$ and any $x\in\mathbb{R}^n$,

$$
|\langle x,z\rangle|
\;=\;
|\langle V^{-1/2}x,\; V^{1/2}z\rangle|
\;\le\;
\|V^{-1/2}x\|_2\;\|V^{1/2}z\|_2
\;=\;
\sqrt{x^\top V^{-1}x}\;\|V^{1/2}z\|_2 .
$$

Taking $z=\widehat\theta-\theta$ and supremum over $x\in\mathcal{X}$ yields

$$
\sup_{x\in\mathcal{X}} |\langle x,\widehat\theta-\theta\rangle|
\;\le\;
\Big(\sup_{x\in\mathcal{X}} \sqrt{x^\top V^{-1}x}\Big)\;\|Z\|_2 .
$$

By the design choice,

$$
\sup_{x\in\mathcal{X}} x^\top V^{-1}x \;\le\; 2 n^2
\quad\Longrightarrow\quad
\sup_{x\in\mathcal{X}} |\langle x,\widehat\theta-\theta\rangle|
\;\le\;
\sqrt{2}\,n\;\|Z\|_2 .
$$

Due to the results in Chapter 20 of \cite{lattimore2020bandit}, we have the following with probability at least $1-\delta$:
\[
\|Z\|_2 \;\le\; 2\sigma \sqrt{\tfrac{2}{N}\Big(n\ln 6 + \ln \tfrac{1}{\delta}\Big)}
\;\le\; 4\sigma \sqrt{\tfrac{n + \ln(1/\delta)}{N}} \,,
\]

Hence, we have

\begin{equation}\label{eq:ols}
\Pr\!\left(
\sup_{x\in\mathcal{X}} |\langle x,\widehat\theta-\theta\rangle|
\;\le\;
6\sigma\sqrt{\frac{n^3+n^2\ln(1/\delta)}{N}}
\right)
\;\ge\; 1-\delta .
\end{equation}

\subsection{Sampling method and estimator}\label{appendix-sec:estimator}
Recall that in the row player’s algorithm, during the $s$-th round of phase $t$, it selects $x_{t,s}$. Analogous to the row player's algorithm, we can describe an algorithm for the column player, where in the $s$-th round of the $t$-th phase it selects $y_{t,s}$. Denote the row player’s true expected utility vector as
\[
\bar{\theta}_t := A \widehat{y}_{t},
\]
where $\widehat{y}_{t} := \mathbb{E}[y_{t,s}]$ is the column player’s averaged strategy in phase $t$.

Recall that $\mathcal{D}_{\mathcal{X}}$ is an uniform distribution over a subset $\{x_1,x_2,\ldots,x_n\}$. Let us fix a sequence of vectors $x_{t,1},x_{t,2},\ldots,x_{t,B_t}$ such that $|\{s:x_{t,s}=\bar{x}_t\}|\geq B_t/4$ and for all $i\in [n]$, $|\{s:x_{t,s}=(1-\lambda_t)\bar{x}_t+\lambda_tx_i\}|\geq B_t/(4n)$. Conditioned on this sequence, we now construct an estimator $\widehat\theta_t^x$ of $\bar\theta_t^x$ with desirable concentration guarantees. 

In each round $s$ of phase $t$, the row player plays $x_{t,s} \in \mathcal{X}$ and receives the reward
\[
r_{t,s} = \langle x_{t,s}, Ay_{t,s}\rangle.
\]
We can decompose this reward as
\[
r_{t,s} = \langle x_{t,s}, A\widehat{y}_{t}\rangle + \langle x_{t,s}, A(y_{t,s} - \widehat{y}_t)\rangle=\langle x_{t,s}, \bar{\theta}_t\rangle+\eta_{t,s},
\]
where the second term $\eta_{t,s}:=\langle x_{t,s}, A(y_{t,s} - \widehat{y}_t)\rangle$ is a zero-mean 4$\lambda_t^2$-subgaussian noise in each phase $t$ due to Lemma \ref{subgaussian-noise-lemma}. Note that the noises $\eta_{t,s}$ are independent and we can correctly apply Lemma \ref{subgaussian-noise-lemma} as the sequence $\{x_{t,s}\}_{s\in [B_t]}$ is fixed.

Let $\{s_1, s_2, \ldots, s_{B_t/2}\}$ be the first $B_t/4$ indices such that $x_{t,s} = \bar{x}_t$. Similarly, let $\{s_1', s_2', \ldots, s_{B_t/4}'\}$ be the set of indices consisting of the first $B_t/(4n)$ indices such that $x_{t,s} = (1-\lambda_t)\bar{x}_t + \lambda_t x_i$ for all $i \in [n]$.

We construct pairs $(s_i, s_i')$ such that
\[
x_{t,s_i'} = (1-\lambda_t)x_{t,s_i} + \lambda_t z_{t,s_i}.
\]
where $z_{t,s_i}\in\{x_1,x_2,\ldots,x_n\}$. Now consider the transformed reward
\[
\widehat{r}_{t,s_i'} := \frac{r_{t,s_i'} - (1-\lambda_t)r_{t,s_i}}{\lambda_t}
= \langle z_{t,s_i'}, \theta_t \rangle + \widehat{\eta}_{t,s_i'},
\]
where $\widehat{\eta}_{t,s_i'}$ is zero-mean $8$-subgaussian noise.

Thus, from the pairs $(z_{t,s_i'}, \widehat{r}_{t,s_i'})$, we obtain an unbiased estimator of $\bar{\theta}_t^x$:
\[
\widehat{\theta}_t^x
= \left(\sum_{i=1}^{B_t/4} z_{t,s_i'} z_{t,s_i'}^\top\right)^{-1}
  \sum_{i=1}^{B_t/4} \widehat{r}_{t,s_i'} z_{t,s_i'}=\left(\sum_{i=1}^{n} x_i x_i^\top\right)^{-1}
  \left(\frac{4}{B_t}\sum_{i=1}^{B_t/4} \widehat{r}_{t,s_i'} z_{t,s_i'}\right).
\]

Due to Eq. \eqref{eq:ols}, conditioning on the sequence $\{x_{t,s}\}_{s\in[B_t]}$ we have the following :
\begin{equation}\label{eq:seq-cond-prob}
\Pr\!\left(
\sup_{x\in\mathcal{X}} |\langle x,\widehat\theta_t^x-\bar{\theta}_t^x\rangle|
\;\le\;
48\sqrt{\tfrac{n^3}{t^3}}
\;\middle|\;\{x_{t,s}\}_{s\in[B_t]}
\right)
\;\ge\; 1-\delta/(8t^2) .
\end{equation}

Now consider a sequence of vectors $x_{t,1},x_{t,2},\ldots,x_{t,B_t}$ generated by our algorithm and define random variables $N_{t,0}:=|\{s:x_{t,s}=\bar{x}_t\}|$ and for all $i\in [n]$, $N_{t,i}:=|\{s:x_{t,s}=(1-\lambda_t)\bar{x}_t+\lambda_tx_i\}|$. Observe that $\mathbb{E}[N_{t,0}]=B_t/2$ and $\mathbb{E}[N_{t,i}]=B_t/(2n)$ for all $i\in [n]$. Consider $i\in [n]\cup\{0\}$. Due to Chernoff bound, for any phase $t$ such that $B_t\geq 32 n^2\ln(8t^2/\delta)$ we have the following
\begin{align*}
    \Pr\left(
N_{t,i}< \mathbb{E}[N_{t,i}]/2)\right)&\leq \exp(-\frac{\mathbb{E}[N_{t,i}]}{8})\\
&\leq \exp(-\frac{B_t}{16n})\tag{as $\mathbb{E}[N_{t,i}]\geq B_t/(2n)$}\\
&\leq \exp(-2n\ln(8t^2/\delta))\tag{as $B_t\geq 32 n^2\ln(8t^2/\delta)$}\\
&\leq \delta/(16nt^2),
\end{align*}
where we get the last step follows from the fact that $x\ln(1/y)\geq \ln(x/y)$ for all $x\geq 2$ and $0<y\leq 1/2$.

Now due to union bound, for any phase $t$ such that $B_t\geq 32 n^2\ln(8t^2/\delta)$ we have following
\begin{equation}\label{eq:seq-chernoff}
\Pr\left(
N_{t,0}\geq B_t/4\text{ and }N_{t,i}\geq B_t/(4n)\; \forall i\in[n]
\right)\geq 1-\delta/(8t^2)
\end{equation}

Hence, due to Eq. \eqref{eq:seq-cond-prob} and Eq. \eqref{eq:seq-chernoff}, we have the following
\begin{equation}\label{eq:final-estimate-prob}
\Pr\!\left(
\sup_{x\in\mathcal{X}} |\langle x,\widehat\theta_t^x-\bar{\theta}_t^x\rangle|
\;\le\;
48\sqrt{\tfrac{n^3}{t^3}}
\right)
\;\ge\; 1-\delta/(4t^2) .
\end{equation}

Note that if $B_t<32 n^2\ln(8t^2/\delta)$, then we set $\widehat \theta^x_t=\mathbf{0}$ and the above inequality holds trivially.




\section{Proofs for primal-dual norms}\label{appendix:norms}
First, we prove that $\|z\|_{*,\mathcal{X}}$ is a norm. 

(i) Positive definiteness. If $z=0$, then $\|z\|_{*,\mathcal{X}}=0$. If $\|z\|_{*,\mathcal{X}}=0$, then $|\langle x,z\rangle|=0$ for all $x\in \mathcal{X}$. As $\mathrm{span}(\mathcal{X})=\mathbb{R}^d$, we have $\langle v,z\rangle=0$ for all $v\in\mathbb{R}^d$. This implies that $\|z\|_2=0$, so $z=0$.

(ii) Absolute homogeneity.
For any scalar $\alpha$,
$$
\|\alpha z\|_{*,\mathcal{X}}=\max_{x\in \mathcal{X}}|\langle x,\alpha z\rangle|
=|\alpha|\max_{x\in \mathcal{X}}|\langle x,z\rangle|
=|\alpha|\,\|z\|_{*,\mathcal{X}}.
$$

(iii) Triangle inequality.
For any $z,w$,
\begin{align*}
\|z+w\|_{*,\mathcal{X}}&=\max_{x\in X}|\langle x,z+w\rangle|\\
&\le \max_{x\in X}\big(|\langle x,z\rangle|+|\langle x,w\rangle|\big)\\
&\le \|z\|_{*,\mathcal{X}}+\|w\|_{*,\mathcal{X}}.
\end{align*}

Thus $\|\cdot\|_{*,\mathcal{X}}$ is a norm.

Next we prove that $\|z\|_{\mathcal{X}}$ is a norm.

(i) Positive definiteness. If $z=0$, then $\|z\|=0$. If $z\neq 0$, then $\|z\|\geq \langle \frac{z}{\|z\|_{*,\mathcal{X}}},z\rangle=\frac{\|z\|_2^2}{\|z\|_{*,\mathcal{X}}}>0$. 

(ii) Absolute homogeneity.
For any scalar $\alpha$,
\begin{align*}
\|\alpha z\|_{\mathcal{X}}&=\max_{\|y\|_{*,\mathcal{X}}\leq 1}\langle y,\alpha z\rangle\\
&= \max_{\|y\|_{*,\mathcal{X}}\leq 1}|\langle y,\alpha z\rangle|\tag{as $\|y\|_{*,\mathcal{X}}=\|-y\|_{*,\mathcal{X}}$}\\
&=|\alpha|\max_{\|y\|_{*,\mathcal{X}}\leq 1}|\langle y, z\rangle|\\
&=|\alpha|\max_{\|y\|_{*,\mathcal{X}}\leq 1}|\langle y, z\rangle|\\
&=|\alpha|\max_{\|y\|_{*,\mathcal{X}}\leq 1}\langle y, z\rangle=|\alpha|\|z\|_{\mathcal{X}}\\
\end{align*}
(iii) Triangle inequality.
For any $z,w$,
\begin{align*}
\|z+w\|_{\mathcal{X}}&=\max_{\|y\|_{*,\mathcal{X}}\leq 1}|\langle y,z+w\rangle|\\
&\le \max_{\|y\|_{*,\mathcal{X}}\leq 1}\big(|\langle y,z\rangle|+|\langle y,w\rangle|\big)\\
&\le \|z\|_{\mathcal{X}}+\|w\|_{\mathcal{X}}
\end{align*}

Thus $\|\cdot\|_{\mathcal{X}}$ is a norm.

Now we show that $\|z\|_{\mathcal{X}}$ and $\|z\|_{*,\mathcal{X}}$ are primal-dual norm pairs. It suffices to show $\max_{x\in\mathcal{X}}|\langle x,z\rangle|=\max_{\|y\|\leq 1}\langle y,z\rangle$ for any $z$. Let $K:=\mathrm{conv}(\mathcal{X}\cup-\mathcal{X})$.

Define $B:=\{y:\|y\|_{\mathcal{X}}\le 1\}$. We begin by showing $B=K$.

(i) $K\subseteq B$.
Consider $y\in K$. Then due to Carathéodory's theorem, there exists a subset $\{x_1,x_2,\ldots, x_\ell\}\subseteq \mathcal{X}$ such that $y=\sum_{i=1}^\ell\lambda_i s_i x_i$ where $\lambda_i\ge0$, $\sum_{i=1}^\ell\lambda_i=1$, $s_i\in\{-1,+1\}$. Now we have the following:
\begin{align*}
    \|y\|_{\mathcal{X}}&=\max_{\|z\|_{*,\mathcal{X}}\leq 1}\langle z,y\rangle\\
    &=\max_{\|z\|_{*,\mathcal{X}}\leq 1}\sum_i\lambda_i s_i\langle z,x_i\rangle\\
    &=\max_{\|z\|_{*,\mathcal{X}}\leq 1}\sum_i\lambda_i |\langle z,x_i\rangle|\\
    &\leq \max_{\|z\|_{*,\mathcal{X}}\leq 1}\sum_i\lambda_i\cdot 1=1
\end{align*}
Hence $y\in B$. Since $y$ was chosen arbitrarily, we have $K\subseteq B$.

(ii) $B\subseteq K$.
Consider $y\notin K$. As $K$ is compact convex, there exists a vector $z$ such that
$$
\langle y,z\rangle \;>\; \max_{x\in K}\langle x,z\rangle,
$$
due to the hyperplane separation theorem.

Set $t:=\max_{x\in K}\langle x,z\rangle=\max_{x\in \mathcal{X}}|\langle x,z\rangle|$. The last equality follows as $K=\operatorname{conv}(\mathcal X\cup -\mathcal X)$. Note that $t>0$, and for $u:=z/t$,

$$
\|u\|_{*,\mathcal{X}}=\max_{x\in \mathcal{X}}|\langle x,u\rangle|=\frac{1}{t}\max_{x\in \mathcal{X}}|\langle x,z\rangle|=1.
$$

Now we have the following:

$$
\|y\|_{\mathcal{X}}=\max_{\|x\|_{*,\mathcal{X}}\leq 1}\langle x,y\rangle\geq \langle u,y\rangle=\frac{\langle z,y\rangle}{t} \;>\; 1,
$$

Thus $y\notin B$. Therefore $B\subseteq K$.

Hence, for any $z$, we have,
$$
\max_{\|y\|_{\mathcal{X}}\le1}\langle y,z\rangle
=\max_{y\in B}\langle y,z\rangle
=\max_{y\in K}\langle y,z\rangle
=\max_{x\in \mathcal{X}}|\langle x,z\rangle|.
$$

\section{Optimistic FTRL Algorithm and RVU Property}\label{appsec:OFTRL_regularizer_proof}

For completeness, we adapt Proposition 7 of \cite{syrgkanis2015fast} to general convex set. 

\oftrl*
\begin{proof}

We start by restating the updates we make with utility sequence $\{u_k\}$, $k \in [t-1]$, and the regularizer $R$.

Lemma 1 in lecture notes \cite{haipeng_lecture_notes} gives
\[
\sum_{t=1}^T\left\langle u_t, x-\tilde{x}_t\right\rangle \leq \frac{D_R\left(x, \tilde{x}_1^{\prime}\right)}{\eta}+\sum_{t=1}^T\left\langle u_t-u_{t-1}, \tilde{x}_{t+1}^{\prime}-\tilde{x}_t\right\rangle-\frac{1}{\eta} \sum_{t=1}^T\left(D_R\left(\tilde{x}_t, \tilde{x}_t^{\prime}\right)+D_R\left(\tilde{x}_{t+1}^{\prime}, \tilde{x}_t\right)\right)
\]
where $\tilde{x}_t^{\prime}$ is a hypothetical "vanilla" FTRL player that doesn't use the optimistic guess $u_{t-1}$. 

Similar to how Theorem 1 is shown in \cite{haipeng_lecture_notes}, we bound the middle term with help of Lemma 4 in Lecture 2 notes of the same lecture note series. 
\[
\left\|\tilde{x}_t-x_{t+1}^{\prime}\right\| \leq \frac{\eta}{\sigma}\left\|\left(\sum_{k=1}^{t-1} u_k+u_{t-1}\right)-\left(\sum_{k=1}^t u_k\right)\right\|_* \leq \frac{\eta}{\sigma}\left\|u_{t-1}-u_t\right\|_*
\]

And via Cauchy-Schwartz step

\begin{align*}
\left\langle u_{t-1}-u_t, \tilde{x}_t-\tilde{x}_{t+1}'\right\rangle & \leq\left\|u_{t-1}-u_t\right\|_* \cdot\left\|\tilde{x}_t-x_{t+1}^{\prime}\right\|\\
&\le \|u_{t-1}-u_t\|_* \cdot \left( \frac{\eta}{\sigma}\|u_{t-1}-u_t\|_* \right) \\
&= \frac{\eta}{\sigma}\|u_t-u_{t-1}\|_{*}^2
\end{align*}

Summing over $t$ and putting everything together
\[
\sum_{t=1}^T\left\langle u_{t-1}-u_t, \tilde{x}_t-\tilde{x}_{t+1}^{\prime}\right\rangle \leq \frac{\eta}{\sigma} \sum_{t=1}^T\left\|u_t-u_{t-1}\right\|_*^2
\]

Now we bound the Bregman terms similar to how it is done in the lecture notes. We first drop the non-negative terms at the boundaries and shift the index to get a lower bound:
\begin{align*}
    \sum_{t=1}^T\left(D_R\left(\tilde{x}_t, \tilde{x}_t^{\prime}\right)+D_R\left(\tilde{x}_{t+1}^{\prime}, \tilde{x}_t\right)\right) &\geq \sum_{t=2}^T\left(D_R\left(\tilde{x}_t, \tilde{x}_t^{\prime}\right)+D_R\left(\tilde{x}_t^{\prime}, \tilde{x}_{t-1}\right)\right)\\
    &\geq \frac{\sigma}{2} \sum_{t=2}^T\left(\left\|\tilde{x}_t-\tilde{x}_t^{\prime}\right\|^2+\left\|\tilde{x}_t^{\prime}-\tilde{x}_{t-1}\right\|^2\right) \quad\quad \text{($R$ is $\sigma$-strongly convex)}\\
    &\geq \frac{\sigma}{4} \sum_{t=2}^T\left(\left\|\tilde{x}_t-\tilde{x}_t^{\prime}\right\|+\left\|\tilde{x}_t^{\prime}-\tilde{x}_{t-1}\right\|\right)^2 \quad\quad \text{($a^2+b^2 \geq(a+b)^2 / 2$)}\\
    &\geq \frac{\sigma}{4} \sum_{t=2}^T\left\|\left(\tilde{x}_t-\tilde{x}_t^{\prime}\right)+\left(\tilde{x}_t^{\prime}-\tilde{x}_{t-1}\right)\right\|^2 \quad \quad \text{(Triangle Inequality)}\\
    &=\frac{\sigma}{4} \sum_{t=2}^T\left\|\tilde{x}_t-\tilde{x}_{t-1}\right\|^2
\end{align*}

Thus we have a upper bound on the negative term of Bregman divergence:
\[
-\frac{1}{\eta} \sum_{t=1}^T\left(D_R\left(\tilde{x}_t, \tilde{x}_t^{\prime}\right)+D_R\left(\tilde{x}_{t+1}^{\prime}, \tilde{x}_t\right)\right) \leq-\frac{\sigma}{4 \eta} \sum_{t=2}^T\left\|\tilde{x}_t-\tilde{x}_{t-1}\right\|^2
\]

Putting everything together and using the fact that $\tilde{x}_1=\tilde{x}_1'$.
\[
 \sum_{t=1}^T\langle u_t,x-\tilde{x}_t\rangle \le \frac{D_R(x,\tilde{x}_1')}{\eta} + \frac{\eta}{\sigma}\sum_{t=1}^T\|u_t-u_{t-1}\|_*^2 - \frac{\sigma}{4\eta}\sum_{t=2}^T\|\tilde{x}_t-\tilde{x}_{t-1}\|^2
\]





\end{proof}

\section{RVU with Estimation Error}\label{appendix:regret-estimation-error}




\regretwithestimationerror*
\begin{proof}
    For the simplicity of presentation, let us $\|\cdot\|$ to denote the primal norm $\|\cdot\|_{\mathcal{X}}$ and $\|\cdot\|_*$ to denote the dual norm $\|\cdot\|_{*,\mathcal{X}}$.
    Recall that $u_t^x=A\tilde{y}_t$ is the true utility of each phase $t$.
    Let $\bar{u}_t^x := t\,\bar\theta_t^x-(t-1)\,\bar\theta_{t-1}^x$ denote the pseudo-utility for phase $t$. Note that $\bar{u}_t^x=\mathbb{E}[\widehat{u}_t^x]$. Let $\mathcal{D}_t$ be a distribution over $\mathcal{X}$ such that with probability $1/2$ we choose $\bar{x}_t$ and with remaining probability we uniformly sample an element from $\{y_1,y_2,\ldots,y_n\}$ and choose it. Recall that $\widehat{y}_t=\mathbb{E}[y_{t,s}]$. Let $\widehat{z}_t:=\mathbb{E}_{z\sim\mathcal{D}_t}[z]$. Observe that $\widehat{y}_t=(1-\lambda_t)\bar{y}_t+\lambda_t\widehat{z}_t$. Recall that $\bar{y}_t=\frac{1}{t}\sum_{\ell=1}^t\tilde{y}_\ell$. Now observe that $\bar{u}_t^x$ can be further simplied as follows:
    \begin{align}
        \bar{u}_t^x &= t\,\bar\theta_t^x-(t-1)\,\bar\theta_{t-1}^x\nonumber\\
        &=t\cdot A\widehat{y}_t-(t-1)\cdot A\widehat{y}_{t-1}\nonumber\\
        &=(1-\lambda_t)A\tilde{y}_t+(\lambda_{t-1}-\lambda_t)\sum_{s=1}^{t-1}A\tilde{y}_s+t\cdot \lambda_t A\widehat{z}_t-(t-1)\cdot \lambda_{t-1} A\widehat{z}_{t-1}\label{st:true-pseudo-ut}
    \end{align}
    Now, we have the following due to triangle inequality:
    \begin{align*}
        \|\bar{u}_t^x-\bar{u}_{t-1}^x\|_*&\leq \|u_t^x-u_{t-1}^x\|_*+\lambda_t\|u_t^x\|_*+\lambda_{t-1}\|u_{t-1}^x\|_*+(\lambda_{t-1}-\lambda_t)\cdot\sum_{s=1}^{t-1} \|u_s^x\|_*+t\cdot \lambda_t \|A\widehat{z}_t\|_*\\
        &\quad +(t-1)\cdot \lambda_{t-1}\|A\widehat{z}_{t-1}\|_*+(\lambda_{t-2}-\lambda_{t-1})\cdot\sum_{s=1}^{t-2} \|u_s^x\|_*+(t-1)\cdot \lambda_{t-1} \|A\widehat{z}_{t-1}\|_*\\
        &\quad +(t-2)\cdot \lambda_{t-2}\|A\widehat{z}_{t-2}\|_*\\
        &\leq \|u_t^x-u_{t-1}^x\|_*+O(1/t),
    \end{align*}
    where the last inequality follows from the fact that $\lambda_t=\frac{1}{t^2}$ and $|\langle x,Ay\rangle|\leq 1$ for all $x\in \mathcal{X}$, $y\in\mathcal{Y}$.

    Next, we have the following due to Eq. \eqref{st:true-pseudo-ut} and the fact that $|\langle x,y\rangle|\leq \|x\|\cdot\|y\|_*$:
    \begin{align*}
        \langle u_t^x,x-\tilde{x}_t\rangle &\leq \langle\bar{u}_t^x,x-\tilde{x}_t\rangle+\lambda_t\|u_t^x\|_*\cdot\|x-\tilde{x}_t\|+(\lambda_{t-1}-\lambda_t)\sum_{s=1}^{t-1}\|u_t^x\|_*\cdot\|x-\tilde{x}_t\|\\
        &\quad+t\cdot \lambda_t \|A\widehat{z}_t\|_*\cdot\|x-\tilde{x}_t\|+(t-1)\cdot \lambda_{t-1}\|A\widehat{z}_{t-1}\|_*\cdot\|x-\tilde{x}_t\|\\
        &\leq \langle\bar{u}_t^x,x-\tilde{x}_t\rangle+2\lambda_t\|u_t^x\|_*+2(\lambda_{t-1}-\lambda_t)\sum_{s=1}^{t-1}\|u_t^x\|_*\\
        &\quad+2t\cdot \lambda_t \|A\widehat{z}_t\|_*+2(t-1)\cdot \lambda_{t-1}\|A\widehat{z}_{t-1}\|_*\tag{as $\|x-\tilde{x}_t\|\leq \|x\|+\|\tilde{x}_t\|\leq 2$}\\
        &\leq \langle\bar{u}_t^x,x-\tilde{x}_t\rangle+O(1/t),
    \end{align*}
     where the last inequality follows from the fact that $\lambda_t=\frac{1}{t^2}$ and $|\langle x,Ay\rangle|\leq 1$ for all $x\in \mathcal{X}$, $y\in\mathcal{Y}$.
     
    Recall that $\Delta^x_t \coloneq \widehat \theta_t - \bar \theta_t$. Now we define $\delta_t:=\widehat{u}_t^x-\bar{u}_t^x=\widehat{u}_t^x-\mathbb{E}[\widehat{u}_t^x]$. Then  
    \begin{align*}
    \delta_t&=t\widehat \theta_t^x - (t-1)\widehat \theta_{t-1}^x - t\bar \theta_t^x + (t-1)\bar \theta_{t-1}^x=t\Delta^x_t-(t-1)\Delta^x_{t-1}
\end{align*}

    Pseudo regret can be written in terms of regret against the estimated utilities and the error term $\delta_t$.
    \[
    \sum_{t=1}^T\left\langle \bar{u}_t^x, x-\tilde{x}_t\right\rangle=\sum_{t=1}^T\left\langle\widehat{u}_t^x, x-\tilde{x}_t\right\rangle-\sum_{t=1}^T\left\langle\delta_t, x-\tilde{x}_t\right\rangle
    \]

    We apply Lemma \ref{lemma:oftrl} with $\sigma:=1$ to the sequence the algorithm actually sees: $\widehat u_t^x$. The lemma above gives a bound on $\sum_{t=1}^T \langle \widehat u_t^x, x - \tilde{x}_t\rangle$:
\[
\sum_{t=1}^T \langle \widehat u_t^x, x- \tilde{x}_t \rangle \le \frac{D_R(x,\tilde{x}_1)}{\eta}
+ \frac{\eta}{\sigma}\sum_{t=1}^T\|\widehat u_t^x-\widehat u_{t-1}^x\|_*^2
- \frac{\sigma}{4\eta}\sum_{t=1}^T\|\tilde{x}_t-\tilde{x}_{t-1}\|^2.
\]

Substituting this back, we get our main inequality:
\begin{align} \label{eq:main_ineq} \tag{1}
\sum_{t=1}^T \langle \bar{u}_t^x, x-\tilde{x}_t \rangle \le \frac{D_R(x,\tilde{x}_1)}{\eta}
& + \underbrace{\frac{\eta}{\sigma}\sum_{t=1}^T\|\widehat{u}_t^x - \widehat{u}_{t-1}\|_*^2}_{\text{Term II}} 
+\underbrace{\sum_{t=1}^T \langle \delta_t, \tilde{x}_t-x \rangle}_{\text{Term I}}
- \frac{\sigma}{4\eta}\sum_{t=1}^T\|\tilde{x}_t-\tilde{x}_{t-1}\|^2
\end{align}

\paragraph{Bounding Term I}
We rewrite Term I using summation by parts.
\begin{align*}
\text{Term I} = \sum_{t=1}^T \langle \delta_t, \tilde{x}_t-x \rangle &= \sum_{t=1}^T \langle t\Delta^x_t - (t-1)\Delta^x_{t-1}, \tilde{x}_t-x \rangle \\
&= \langle T\Delta^x_T, \tilde{x}_T-x \rangle + \sum_{t=1}^{T-1} \langle t\Delta^x_t, \tilde{x}_t - \tilde{x}_{t+1} \rangle \\
&\le 2\|T\Delta^x_T\|_* \|\tilde{x}_T-x\| + \sum_{t=1}^{T-1} \|t\Delta^x_t\|_* \|\tilde{x}_t - \tilde{x}_{t+1}\|
\end{align*}

Since $x,\tilde{x}_t \in \cX$, the term $\|\tilde{x}_T-x\| \le 2\sup_{z\in \cX} \|z\|\leq 2$. We get the last inequality due to the fact that $\sup_{z \in \cX}\|z\|\leq 1$. Combining the bound above with negative movement term from \eqref{eq:main_ineq}, and using a separate $\frac{\sigma}{16\eta}$ portion for this bound, we get for this sum part:
\[
\sum_{t=1}^{T-1}\left\|t \Delta^x_t\right\|_*\left\|\tilde{x}_t-\tilde{x}_{t+1}\right\|-\frac{\sigma}{16 \eta} \sum_{t=2}^T\left\|\tilde{x}_t-\tilde{x}_{t-1}\right\|^2
\]

Using Young's inequality, $ab \le \frac{a^2}{2c}+\frac{cb^2}{2}$, with $a=\|t\Delta^x_t\|_*$, $b= \|\tilde{x}_t - \tilde{x}_{t+1}\|$, and $c=\frac{\sigma}{8\eta}$,
\[
\left\|t \Delta^x_t\right\|_*\left\|\tilde{x}_t-\tilde{x}_{t+1}\right\| \leq \frac{\left\|t \Delta^x_t\right\|_*^2}{2(\sigma / 8 \eta)}+\frac{(\sigma / 8 \eta)\left\|\tilde{x}_t-\tilde{x}_{t+1}\right\|^2}{2}=\frac{4 \eta}{\sigma}\left\|t \Delta^x_t\right\|_*^2+\frac{\sigma}{16 \eta}\left\|\tilde{x}_t-\tilde{x}_{t+1}\right\|^2
\]

Summing this from $t=1$ to $T-1$, the movement terms cancel the $-\frac{\sigma}{16 \eta}\sum_{t=2}^T\|\tilde{x}_t - \tilde{x}_{t-1}\|^2$ term, leaving on second-order error terms. Thus the contribution from Term I is bounded by:
\[
\text{Term I} -\frac{\sigma}{16 \eta}\sum_{t=2}^T\|\tilde{x}_t - \tilde{x}_{t-1}\|^2  \leq 2\|T\Delta^x_T\|_*  + \frac{4\eta}{\sigma}\sum_{t=1}^{T-1}\|t \Delta^x_t\|_{*}^2
\]

\paragraph{Bounding Term II} We use the inequality $(a+b)^2 \leq 2a^2 + 2b^2$ and the triangle inequality.

\begin{align*}
\|\widehat{u}_t^x - \widehat{u}_{t-1}^x\|_*^2 &= \|(\bar{u}_t^x - \bar{u}_{t-1}^x) + (\delta_t - \delta_{t-1})\|_*^2 \\
&\le 2\|\bar{u}_t^x-\bar{u}_{t-1}^x\|_*^2 + 2\|\delta_t - \delta_{t-1}\|_*^2\\
&\le  4\|u_t^x-u_{t-1}^x\|_*^2 + 4\|\delta_t\|_{*}^2+4 \|\delta_{t-1}\|_{*}^2+O(1/t^2)
\end{align*}

\begin{align*}
\sum_{t=1}^T (\left\|\delta_t\right\|_*^2 + \left\|\delta_{t-1}\right\|_*^2) 
&\le 2 \sum_{t=1}^T\left\|\delta_t\right\|_*^2\\
&\le 4 \sum_{t=1}^T \left\|t\Delta^x_t\right\|_*^2+4\sum_{t=1}^T  \| (t-1)\Delta^x_{t-1}\|_*^2\\
&\le 8 \sum_{t=1}^T \left\|t\Delta^x_t\right\|_*^2
\end{align*}

\begin{align*}
\text{Term II}=\frac{\eta}{\sigma}\sum_{t=1}^T \|\widehat{u}_t^x - \widehat{u}_{t-1}\|_*^2\le \frac{\eta}{\sigma} \left(\sum_{t=1}^T 4\|u_t^x-u_{t-1}^x\|_*^2 + 32\sum_{t=1}^T\|t\Delta^x_t\|_*^2\right) +O\left(\frac{\eta}{\sigma}\right)
\end{align*}

\paragraph{Combining the bounds.}

\begin{align*}
\sum_{t=1}^T \langle \bar{u}_t^x, x-\tilde{x}_t \rangle &
\le \frac{D_R(x,\tilde{x}_1)}{\eta} + \text{Term II} + \text{Term I} - \frac{\sigma}{4\eta}\sum_{t=1}^T\|\tilde{x}_t-\tilde{x}_{t-1}\|^2 \\
& = \frac{D_R(x,\tilde{x}_1)}{\eta} + \text{Term II} + \left(\text{Term I} - \frac{\sigma}{16\eta}\sum_{t=2}^T\|\tilde{x}_t-\tilde{x}_{t-1}\|^2\right) \\ & \quad - \frac{\sigma}{4\eta}\|\tilde{x}_1-x_0\|^2 - \left(\frac{\sigma}{4\eta} - \frac{\sigma}{16\eta}\right)\sum_{t=2}^T\|\tilde{x}_t-\tilde{x}_{t-1}\|^2 \\
& \leq \frac{D_R(x,\tilde{x}_1)}{\eta} + 2\|T\Delta^x_T\|_* \sup_{z \in \mathcal{X}}\|z\| + \frac{36\eta}{\sigma}\sum_{t=1}^{T}\|t \Delta^x_t\|_{*}^2 + \frac{4\eta}{\sigma} \sum_{t=1}^T \|u_t^x-u_{t-1}^x\|_*^2 \\ & \quad - \frac{\sigma}{4\eta}\|\tilde{x}_1\|^2 - \frac{3\sigma}{16\eta}\sum_{t=2}^T\|\tilde{x}_t-\tilde{x}_{t-1}\|^2+O\left(\frac{\eta}{\sigma}\right)
\end{align*}

Since $\sigma/4\eta > 3\sigma/16\eta$, we can weaken the bound on the $\|\tilde{x}_1\|^2$ term to get a single, compact sum:
$$- \frac{\sigma}{4\eta}\|\tilde{x}_1\|^2 - \frac{3\sigma}{16\eta}\sum_{t=2}^T\|\tilde{x}_t-\tilde{x}_{t-1}\|^2 \le - \frac{3\sigma}{16\eta}\sum_{t=1}^T\|\tilde{x}_t-\tilde{x}_{t-1}\|^2$$
This gives the final result:

\begin{align*}
\sum_{t=1}^T \langle u_t^x, x-\tilde{x}_t \rangle \leq \frac{D_R(x,\tilde{x}_1)}{\eta} &+ 2\|T\Delta^x_T\|_*  + \frac{36\eta}{\sigma}\sum_{t=1}^{T}\|t \Delta^x_t\|_{*}^2 \\ &+ \frac{4\eta}{\sigma} \sum_{t=1}^T \|u_t^x-u_{t-1}^x\|_*^2 - \frac{3\sigma}{16\eta}\sum_{t=1}^T\|\tilde{x}_t-\tilde{x}_{t-1}\|^2+O\left(\frac{\eta}{\sigma}+\log T\right)
\end{align*}

\end{proof}

\end{document}